\documentclass{article}

     \PassOptionsToPackage{numbers, compress}{natbib}

    \usepackage[preprint]{neurips_2022}

\usepackage[utf8]{inputenc} 
\usepackage[T1]{fontenc}    
\usepackage{hyperref}       
\usepackage{url}            
\usepackage{booktabs}       
\usepackage{amsfonts}       
\usepackage{amsthm}
\usepackage{nicefrac}       
\usepackage{microtype}      
\usepackage{xcolor}         
\usepackage{times}
\usepackage{amsmath}
\usepackage{mathtools}
\usepackage{amssymb}
\usepackage{eucal}
\usepackage{natbib}
\usepackage{enumitem}
\usepackage{bm}
\usepackage{graphicx}
\usepackage{mwe}

\title{The Parametric Stability of Well-separated Spherical Gaussian Mixtures}

\author{%
  Hanyu Zhang  \\
  Department of Statistics\\
  University of Washington\\
  Seattle, WA 98115 \\
  \texttt{hanuyz6@uw.edu} \\
   \And
  Marina Meil\u{a} \\
  Department of Statistics\\
  University of Washington\\
  Seattle, WA 98115 \\
  \texttt{mmp@stat.washington.edu} \\
}

\newenvironment{itemize*}{
\begin{itemize}
\setlength{\parskip}{0em}
\setlength{\topparskip}{0em}
}
{\end{itemize}}

\newcommand{\comment}[1]{}
\newcommand{\ssam}[1]{}  
\newcommand{\mmp}[1]{}
\newcommand{\mmpmmp}[1]{}

\newcommand{\beq}{\begin{equation}}
\newcommand{\eeq}{\end{equation}}
\newcommand{\beqa}{\begin{eqnarray}}
\newcommand{\eeqa}{\end{eqnarray}}
\newcommand{\beqas}{\begin{eqnarray*}}
\newcommand{\eeqas}{\end{eqnarray*}}

\newcommand{\bit}{\begin{itemize}}
\newcommand{\eit}{\end{itemize}}
\newcommand{\bits}{\begin{itemize*}}
\newcommand{\eits}{\end{itemize*}}
\newcommand{\benum}{\begin{enumerate}}
\newcommand{\eenum}{\end{enumerate}}
\newcommand{\benums}{\begin{enumerate*}}
\newcommand{\eenums}{\end{enumerate*}}

\newcommand{\sphgmm}{sGMM}

\newcommand{\norm}[1]{\lvert\lvert{#1}\lvert\lvert}

\newcommand{\rrr}{{\mathbb R}}

\newcommand{\M}{{\cal M}}

\newcommand{\pimin}{\pi_{\min}}
\newcommand{\pimax}{\pi_{\max}}
\newcommand{\model}{{\mathcal M}}
\newcommand{\myperm}{\text{\rm perm}}

\newtheorem{lemma}{Lemma}
\newtheorem{definition}[lemma]{Definition}
\newtheorem{proposition}[lemma]{Proposition}
\newtheorem{theorem}[lemma]{Theorem}
\newtheorem{corollary}[lemma]{Corollary}

\newcommand{\theoremref}[1]{Theorem \ref{#1}}
\newcommand{\sectionref}[1]{Section \ref{#1}}
\newcommand{\figureref}[1]{Figure \ref{#1}}
\newcommand{\lemmaref}[1]{Lemma \ref{#1}}

\begin{document}

\maketitle

\begin{abstract}
  We quantify the parameter stability of a spherical Gaussian Mixture
  Model (sGMM) under small perturbations in
  distribution space. Namely, we derive the first explicit bound to show that for a mixture of spherical Gaussian $P$ (\sphgmm) in a pre-defined model class, all other \sphgmm \ close to $P$ in this model class in total variation distance has a small parameter distance to $P$. Further, this upper bound only depends on $P$. The motivation for this work lies in providing guarantees for fitting Gaussian mixtures; with this aim
  in mind, all the constants involved are well defined and
  distribution free conditions for fitting mixtures of spherical Gaussians. Our results tighten considerably the existing computable bounds, and asymptotically match the known sharp thresholds for this problem.
\end{abstract}

\section{Introduction}
We consider the problem of fitting spherical Gaussian Mixture Models (\sphgmm) to an unknown distribution $Q$. Without assuming knowledge about the target $Q$, what kind of guarantees can we give about an estimated \sphgmm\comment{ denoted by $P$}? And under what condition are guarantees possible?

Previous work (e.g.,\citep{Dasgupta1999,Arora2001,Achlioptas2005,Kannan2008}, etc) established estimation guarantees for the  mixture parameters under model assumptions about the data source $Q$ (being, e.g., a mixture of well separated Gaussians). Specifically, W.r.t. the scope of this paper, when $Q$ is a mixture of $K$ spherical Gaussians, satisfying two criteria: (i) non-vanishing component proportions, and (ii) sufficient separation between components, polynomial run-time estimation algorithms exist.

Our work asks the question: what can be said without such precise knowledge of $Q$? We assume instead {\em indirect} knowledge of $Q$, namely that a well-separated mixture model $P$ was fit to $Q$, and that {\em the model fit is good}. The main difference is that now $Q$ is not required to belong to the model class, but to be ``close'' to it, in a way to be defined (specifically, in this paper, the model fit is measured by the {\em total variation distance} between $P$ and $Q$, $TV(P,Q)$). We also assume that $P$ is a mixture of well-separated Gaussians, with component proportions bounded below.  As it turns out, knowing that $Q$ is close to a ``good'' model class, is almost as useful as knowing $Q$ is in the model class. Under these conditions on $P$ and $Q$, we prove that $P$'s parameters are unique up to perturbations that we upper bound. 
\comment{But with respect to a broader goal of validating model $P$ from data, the difference is that goodness of fit can be estimated, while the }
\mmp{maybe for journal: in statistics, this is post-estimation??  inference, guarantees under model missspecification, ... We also assume that
the model fit is good, as measured by the {\em total variation distance} $TV$ between $P$ and $Q$, and}
In summary, we aim to prove a statment the following form.
\begin{theorem}[Generic $(\epsilon,\delta)-$Stability]  \label{thm:generic}
 For distribution $Q$ and a class $\mathcal{M}$ of \sphgmm, if a given model ${P}\in\model$ satisfies $TV(Q,{P})\leq \epsilon$, then, for any $P'\in \model$  such that $TV(Q,P')\leq \epsilon$, it holds that
 $d_{param}({P},P')\leq \delta$, where $d_{param}(\cdot,\cdot)$ is a divergence defined in parameter space .
\end{theorem}
This type of statement was proposed earlier by \citep{Meila2006} in the context of loss-based clustering, but to date it has not been instantiated for any parametric model fitting problem.  
\mmp{Later: --related}

A distribution ${P}$ that satisfies \theoremref{thm:generic} will be called stable. Obviously, such a distribution is generally not unique or distinguished in $\model$, hence stability should be defined as a property of (a subset of) $\model$. This also follows from setting $Q=P^*\in\model$. Therefore, we define parametric stability of a model class as follows.
\begin{definition}\label{def:2eps-generic} Let $\model$ be a class of spherical Gaussian Mixtures and $d_{param}(\cdot,\cdot)$ be a divergence between parameters; for sufficiently small $\epsilon > 0$, $P$ is $(\epsilon,\delta)-$stable in $\M$ if any model $P'\in\M$ such that $d_{TV}(P,P')\leq 2\epsilon$ satisfies that $d_{param}(P,P')\leq \delta$.
\end{definition}
The technical contribution of our paper is to formulate conditions on $\model$ and establish upper bounds  $\delta(\epsilon,\model)$ for any $P\in\model$, in the population regime. These are given in \theoremref{thm:main1}.  The upper bounds on $\delta(\epsilon,\model)$ we obtain are {\em tractable}, with explicit constants depending on the model class only. As a statistical procedure, the bound $\delta(\epsilon,\model)$ provides quantitative post-estimation evaluation or diagnostic analysis for fitting a Gaussian Mixture Model, without any prior knowledge. 
For example, for an arbitrary population $Q$ (with density $q$) on $\rrr^d$, let $\hat{P}=\sum_{i=1}^{K}\hat{\pi}_iN_d(\hat{\mu}_k,\hat{\sigma}_k^2I_d)$ be a learned \sphgmm. \figureref{fig:example1} shows an example of stable and unstable distribution \sphgmm s, and further for the stable \sphgmm~it constructs good-fit region such that all \sphgmm~within small total variation distance to the population must have parameters located in these regions. These regions can be reported as the usual confidence regions are reported in a regular parametric estimation problem. 

\begin{figure}[!t]
 \centering
 \includegraphics[width=0.9\textwidth]{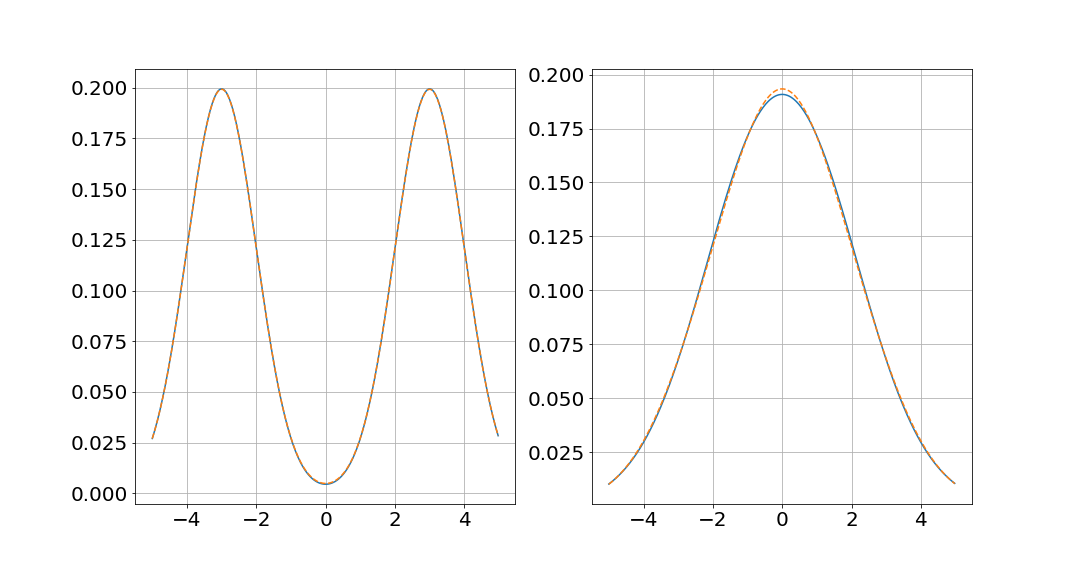}
 \comment{How about this: P_1 = 0.5 N(-3,1) + 0.5 (3,1); P_2 = 0.4995 N(-3,1) + 0.4995 N(3,1) + 0.001 N(0,1)
   TV(P_1,P_2) <= 0.001, and our bounds show that any good 2-mixture fit to P_2, within total variation distance 0.001,  is close to P_1: each component is within 0.0100 max{\sigma_i,\sigma_i'} and max{sigma_i/\sigma_i', \sigma_i'/\sigma} <= 1.0055, and |pi_i-pi_i'|<= 0.02 density plotted here. It in fact has two densities, one dash line (P_2) and one solid line (P_1) but they are covering each other.}
 \comment{P_1 = 0.0625 N(-3,2.25) + 0.4375 N(-1,2.25) + 0.4375 N(1, 2.25) + 0.0625 N(3, 2.25), variances all = 2.25, mean = [-3,-1,1,3].P_2 = 0.0078125 N(-4,2.25) + 0.21875 N(-2,2.25)  + 0.546875 N(0,2.25) + 0.21875 N(2,2.25) + 0.0078125 N(4,\sigma^2), variances all = 2.25, mean = [-4,-2,0,2,4]}
\caption{Stable and unstable spherical Gaussian Mixtures in 1D. {\bf
     Left:} A well separated mixture of $K=2$ Gaussians, $P=
   0.5N(-3,1) + 0.5N(3,1)$, with separation $c=3$ (as defined in
   Section \ref{sec:setup}, A3), superimposed on a distribution $Q$
   such that $TV(P,Q)\leq 0.001$. Our \theoremref{thm:main1} in
   \sectionref{sec:setup} guarantees that any mixture $P'$ with $K=2$ components, minimal weights 0.45, and separation constants 3 that fits $Q$ equally well must have parameters close to
   $P$; namely (see \sectionref{sec:setup} for the parameter
   definitions) $\mu'_{i}$ within $0.0200$ of $-3$ and $3$,
   $\max\{1/\sigma_i'^2, \sigma_i'^2\} \leq 1.034$, and $|0.5-\pi_i'|\leq
   0.004$. {\bf Right:} Two spherical Gaussian mixtures $P,P'$ which
   are unstable, in the sense that while they are close in TV
   distance, they have very different parameters; $P= 0.0625
   N(-3,\sigma^2) + 0.4375 N(-1,\sigma^2) + 0.4375 N(1, \sigma^2) +
   0.0625 N(3, \sigma^2)$, $P'=0.0078125 N(-4,\sigma^2) + 0.21875
   N(-2,\sigma^2) + 0.546875 N(0,\sigma^2) + 0.21875 N(2,\sigma^2)
   + 0.0078125 N(4,\sigma^2)$, with $\sigma^2=2.25$. In this example, the
   mixture components are less separated, and some of the mixture
   proportions are small as well.}
 \label{fig:example1}
\end{figure}
Furthermore, Definition \ref{def:2eps-generic} concerns only the model class $\M$, a subset of spherical Gaussian mixtures. Hence, one can obtain distribution free stability results generically, by studying the stability of distributions inside a model class, such as that of \sphgmm. While this remark is nearly obvious, suprisingly little work has attended to the possiblity of obtaining distribution free guarantees as a side effect of consistency or identifiability proofs. We hope that one contribution of this paper be at the conceptual level, in drawing attention to this possibility, which remains the primary motivator for this work.

In \sectionref{sec:setup} we define model classes $\model$ of
interest, and instantiate $d_{param}(\cdot,\cdot)$ and all other parameters, then
state our main stability result in
\theoremref{thm:main1}. \sectionref{sec:num} illustrates the result with some
numeric examples. In the supplementary materials, \sectionref{sec:proof-main} provides the detailed proof
of our main result \theoremref{thm:main1}, with additional lemmas proved in \sectionref{sec:proof-aux}.

\section{Problem Formulation and Main Results}
\label{sec:setup}
We first pose the problem in formal terms. Then, we state the main result in \theoremref{thm:main1}, namely that if  two well-separated mixtures of spherical Gaussians are close as distributions, in total variation distance, their parameters are also close. 

\subsection{Problem Setup}
\label{sec:setup-1}
In the problem formulation, we identify three key components: (i)
 model class $\mathcal{M}$, (ii)  distance or divergence between the models in parameter space $d_{param}(P,P')$, and (iii) goodness of fit measure.

\paragraph{Model Class} A spherical Mixture of Gaussians $P$ over $\rrr^d$ can be written in the form 
\begin{equation}
 P=\sum_{k=1}^K \pi_k N_d(\mu_k,\sigma_k^2 I_d),\quad\text{with } \sum_{k=1}^K \pi_k = 1,\quad\pi_k\geq 0.
\end{equation}
In the above, we have adopted the standard notation, whereby $N_d(\mu_k,\sigma_k^2I_d)$, called {\em mixture components}, are normal distributions with means $\mu_{1:K}\in\rrr^d$ and diagonal covariance matrices $\sigma_{1:K}^2I_d\in\rrr^{d\times d}$, while $\pi_{1:K}$ are called {\em mixture proportions}.

Further on, we assume the number of components $K$ is fixed, and we add restrictions on the smallest component proportion and the component separation. Thus the model classes we consider are denoted $\mathcal{M}(K,\pi_{\min},\pimax,c)$.
\begin{definition}An \sphgmm~$P$ is in model class $\mathcal{M}(K,\pi_{\min},\pi_{\max},c)$ iff the following holds:
  \begin{itemize}\setlength{\parskip}{0em}\setlength{\topskip}{0em}
    \item[A1] $K\geq 2$
    \item[A2] $\min_{k\in[K]}\pi_k \geq \pi_{\min}$, $\max_{k\in[K]}\pi_k \leq \pi_{\max}$
    \item[A3] $\norm{\mu_i-\mu_j}> c(\sigma_i+\sigma_j)$ for any $i,j\in [K]$ and $i\neq j$
  \end{itemize}
\end{definition}
\mmp{something about conditions -- how standard -- if there is time}
For $P\in\mathcal{M}(K,\pi_{\min},\pi_{\max},c)$, the maximal proportion is always less than or equal to $1-(K-1)\pi_{\min}$. Therefore, we abbreviate the model class notation to be $\mathcal{M}(K,\pi_{\min},c)$ when $\pi_{\max}= 1-(K-1)\pi_{\min}$. The separation constant $c$ is necessary. Theorem 3.1 in \citep{Regev2017} points out that as $K$ increase, when separation constant $c$ is $o(\sqrt{\log K})$, it is possible to find two \sphgmm s so that their parameters are $\Omega(1)$ different but their total variation distance is superpolynomially small in $K$. We rule out this difficulty by deriving assumptions on the separation constant $c$, which regularizes the model class.
\paragraph{The divergence of model parameters} between any two models $P,P'$ with $K$ components is
\begin{align} \label{eq:dparameter}
 d_{param}(P,P') = \min_{\myperm \in \mathbb{S}_K}\max_{k\in[K]} \frac{|\pi_i-\pi_{\myperm(i)}'|}{\min\{\pi_i,\pi_{\myperm(i)}'\}} + \frac{\norm{\mu_i-\mu_{\myperm(i)}'}}{\max\{\sigma_{i},\sigma_{\myperm(i)}'\}} + \frac{|\sigma_i^2-\sigma_{\myperm(i)}'^2|}{\min\{\sigma_i^2,\sigma_{\myperm(i)}'^2\}}
\end{align}
\noindent In the above, $\myperm \in \mathbb{S}_K$ is a permutation of the set $[K]$; $d_{param}$ is not necessarily a metric\footnote{$d_{param}$ is also not a divergence in the information geometric sense.}. \theoremref{thm:main1} presented below provides upper bounds on each of the three terms of \eqref{eq:dparameter} separately. 

Consider $G=\sum_{k=1}^K\pi_k\delta_{(\mu_k,\sigma_k^2)}$ with $\delta$ being the point mass. One can also view each \sphgmm \ distribution $P_G=\sum_{k=1}^K\pi_k N(\mu_k,\sigma_k^2I_d)$ as a smoothed distribution of discrete probability on the parameter space. Let $G,G'$ be two such distributions on the parameter space, then Wasserstein-1 distance defined by 
\begin{equation}
  W_1(G,G')=\inf_{\Pi}\int d((\mu_k,\sigma_k^2),(\mu_k',\sigma_k'^2))\Pi(G,G')
\end{equation}
is considered. This is shown to be a natural parameter divergence in several previous papers (e.g.,\citep{HoNguyen2016,Heinrich2018,Yihong2020,doss2020optimal}). Our $d_{param}$ does not capture the same topology as the $W_1$ distance on parameters. Consider two sequences $P_n=0.5N(-1,(n+1)^2)+0.5N(1,(n+1)^2)$ and $Q_n=0.5N(-1,n^2)+0.5N(1,n^2)$. Then, as $n\rightarrow \infty$, \mmpmmp{$TV(P_n,Q_n)$ and} $d_{param}(P_n,Q_n)$ converges to zero, but $W_1(n^2,(n+1)^2)$ diverges. On the other hand it can also be shown that for a fixed $G$ and corresponding $P_G$, when $W_1(G,G')\rightarrow 0$, it holds that
\begin{equation}
  d_{param}(P_G,P_{G'})\asymp W_1(G,G')\asymp \min_{\myperm\in\mathbb{S}^K}\sum_{i=1}^K |\pi_i-\pi_{\myperm(i)}|+\norm{\mu_i-\mu_{\myperm(i)}'}+|\sigma_i^2- \sigma_{\myperm(i)}'^2|\;.
\end{equation}
 Hence our definition of $d_{param}$ is in better agreement with the model fit criterion, does not need the assumption that the parameter space is compact \citep{HoNguyen2016} and there will be a direct way to construct good-fit regions or confidence regions since we provide upper bounds on the terms in $d_{param}$ separately. Also,  $d_{param}$ is invariant w.r.t. rescaling, and this will be useful in, e.g. model-based clustering.

\paragraph{Goodness of fit} can be measured by various
criteria (e.g. likelihood or KL divergence, $W_2$ distance), but in
this paper total variation distance \citep{Gibbs2002probDist} is used for its convenience. For
two probability distributions $P,Q$ on $\mathbb{R}^d$, \mmpmmp{having densities
$p,q$ respectively w.r.t. Lebesgue measure $\lambda$, }the {\em total
  variation distance} is given by
\begin{equation}
 TV(P,Q)=\sup_{A\in\mathcal{B}}|P(A)-Q(A)|\;,
 \label{def:tv}
\end{equation}
where $\mathcal{B}$ is all Borel sets on $\rrr^d$.\mmpmmp{It is easy to see that by selecting the set $A=\{x:p(x)\geq q(x)\}$, it holds that $P(A)-Q(A)=TV(P,Q)$. This fact will be useful for our proofs.}

\subsection{Preparation for main result: defining the constants}
\mmp{better title needed} Before we formulate \theoremref{thm:generic}
for the specific case of mixtures of spherical Gaussian, we introduce
several constants depending on the model class $\model(K,\pimin,c)$
and $\epsilon$, but independent of the actual distributions
considered. These parametrize the upper bounds on the terms of
$d_{param}$ \eqref{eq:dparameter} we are about present. The
constants $\eta_0,\eta^*$, to be defined below, represent ratios between standard deviations; $\eta^*>1$ will bound the perturbation in $\sigma_i$, and the aim is to bring it as close to 1 as possible. The constant $c_0$, together with $\eta_0$ gives the minimum sufficient value for the separation $c$, while $c^*$ bounds the perturbation of the components' means. 

Denote by $\Phi(x)$ the CDF of the standard normal distribution in one
dimension and by $F_d$ the CDF of Gamma($\frac{d}{2},\frac{d}{2}$)
distribution.\mmpmmp{inline?}
\begin{align}
 \Phi(x) = \int_{-\infty}^x \frac{1}{\sqrt{2\pi}}\exp\left(-\frac{1}{2}t^2\right)dt\;,\quad
 F_d(x) = \int_0^x \frac{(\frac{d}{2})^{\frac{d}{2}}}{\Gamma(\frac{d}{2})}t^{\frac{d}{2}-1}\exp\left(-\frac{d}{2}t\right) dt\;.
\end{align}
Constants $c_0,\eta_0$ determined by $\pi_{\min},\epsilon, \pi_{\max}$(or $K$) are defined by
\begin{align}
 c_0 = 2\Phi^{-1}(1-\frac{\pi_{\min}-2\epsilon}{2})\;;
 \label{eq:c0}
\end{align}
$\eta_0$ satisfies $\eta_0\geq 1$ and 
\begin{equation}
  1-\frac{\pi_{\min}-2\epsilon}{\pi_{\max}}+\frac{2(1-\pi_{\max})}{\pi_{\max}}\Phi(-\frac{1}{2}\eta_0 c_0) = F_d(\frac{2\eta_0^2\log \eta_0}{\eta_0^2-1}) - F_d(\frac{2\log \eta_0}{\eta_0^2-1})\;.
  \label{eq:eta0-diff}
 \end{equation}
\lemmaref{lem:est-eta-function} (Supplement) will show that the right hand side of \eqref{eq:eta0-diff} is in fact $TV(N(0,I_d),N(0,\eta^2I_d))$\comment{the total variation distance between two $d-$dimensional spherical Gaussians with same center and standard deviation ratio being $\eta$} and is lower bounded by $1-(2\eta/(\eta^2+1))^{d/2}$.

With the values $c_0,\eta_0$ one can establish upper bounds for the three terms of equation \eqref{eq:dparameter}; in particular  $c_0\max\{\sigma_i,\sigma'_{\myperm(i)}\}$ bounds the first term, representing the means' variation. This is shown in detail in Section \ref{sec:initial}. We call $c_0,\eta_0$ {\em initial bounds} because, in fact, they can be further reduced, by a technique that we present in \sectionref{sec:refined}. Thus, we obtain {\em refined} parameter distance upper bounds depending on constants $c^*,\eta^*$ which satisfy $0 \leq c^* \leq c_0$, $1\leq \eta^*\leq \eta_0$ and the following conditions
\begin{align}
 1-2\Phi(-\frac{c^*}{2})&=\frac{2\epsilon}{\pi_{\min}}+\frac{2(1-\pi_{\min})}{\pi_{\min}}\Phi(-\frac{1}{2}[c(1+\frac{1}{\eta^*})-c^*])\;,
 \label{eq:cstar} \\
 F_d(\frac{2(\eta^{*})^2\log \eta^{*}}{(\eta^{*})^2-1}) - F_d(\frac{2\log \eta^*}{(\eta^{*})^2-1})&=\frac{2\epsilon}{\pi_{\min}}+\frac{2(1-\pi_{\min})}{\pi_{\min}}\Phi(-\frac{1}{2}[c(1+\frac{1}{\eta^*})-c^*])\;. 
 \label{eq:etastar}
\end{align}
Note that through equations \eqref{eq:cstar} and \eqref{eq:etastar}, $c^*,\eta^*$ are also determined by $\pi_{\min},\epsilon$, $c$ and $\pi_{\max}$. The following section shows that the four constants $c_0,\eta_0,c^*,\eta^*$ exist and are unique. \comment{Further it provides a rought estimates of the four constants under the asymptotic regime that $K\rightarrow \infty$, $\pi_{\min}-2\epsilon$ decreases polynomially with $1/K$.}

\mmp{but only the constant, not the $\sqrt{K}$ part. shall we put this outside the remark?}
\subsection{Main Theorem: Good Fits of \sphgmm \ Can Only Be Close to a Good Fit.}
\label{sec:setup-main1}
\mmp{In particular, a good model fit in TV can be used to perform model selection}
Now we are ready to state our main result. 
\begin{theorem}
Let $P\in\mathcal{M}(K,\pi_{\min},\pi_{\max},c)$. Suppose $P'$ is any model in $\mathcal{M}(K',\pi_{\min},\pi_{\max},c)$ such that $TV(P,P')\leq 2\epsilon$ where $\max\{K,K'\}\leq 1/\pi_{\min}$, $\pi_{\max}\leq 1-(\min\{K,K'\}-1)\pi_{\min}$. Let $c_0,\eta_0$ be defined as in \eqref{eq:c0} and \eqref{eq:eta0-diff}. Then, if $c\geq c_0\eta_0$ and
 $\pi_{\min}>2\epsilon$, we have $K=K'$ and further, there exists a permutation $\myperm\in\mathbb{S}_K$ and constants $c^*\in[0,c_0],  \eta^*\in [1,\eta_0]$ satisfying
 \eqref{eq:cstar} and \eqref{eq:etastar}, such that
 for each $i\in [K]$,
 \begin{align}
 \norm{\mu_i-\mu_{\myperm(i)}'} &\leq c^*\eta^*\sigma_i\\
  \max\{{\sigma_i}/{\sigma_{\myperm(i)}'}, {\sigma_{\myperm(i)}'}{\sigma_i}\}&\leq \eta^* \\
 |\pi_i-\pi_{\myperm(i)}'|&\leq {2\epsilon}+ (1-\pi_{\min}+\pi_{\max}) \Phi(-C(c,c^*,\eta^*))\;,\label{eq:pibound}
 \end{align}
 where $C(c,c^*,\eta^*)$ is defined by 
 \begin{equation}
 C(c,c^*,\eta^*):= \sqrt{\frac{c^2}{2(\eta^*)^2}+\frac{1}{2\eta^*}(c-\frac{c^*}{2})^2-\frac{(c^*)^2(1+\eta^*)^2}{16(\eta^*)^2}}-\frac{c^*}{2}\;.
 \label{eq:dist-delta-pi}
 \end{equation}
 \label{thm:main1}
\end{theorem}
This theorem extends the usual identifiability result. See \sectionref{sec:related} for more discussion with previous works. Furthermore, this upper bound is tractable since all the constants are explicit or computable and it does not assume prior knowledge on parameters. Computability opens up the possibility of applying our result to finite samples. 

The following proposition gives an estimate of the constants $c_0,\eta_0,c^*,\eta^*$ in the asymptotic regime that $K\rightarrow \infty$.
\begin{proposition}
  Given $\pi_{\min}>2\epsilon > 0$, then $c_0,\eta_0,c^*,\eta^*$ defined in equations \eqref{eq:c0}--\eqref{eq:cstar} exist and each is unique. Further suppose there are nonnegative constants $\alpha,\beta$ and sufficiently large positive constant $\gamma$ such that $\beta \leq 1+\alpha$, $\pi_{\min}=\Omega(K^{-(1+\alpha)})$ ,$\pi_{\max}/\pi_{\min}=O(K^\beta)$ and $2\epsilon/\pi_{\min}\leq K^{-\gamma}$. When $K\rightarrow \infty$, we have the initial separation condition $c_0=O(\sqrt{\log K})$, $\eta_0 = O(K^{2\beta/d})$. With any separation $c>c_0\eta_0$, the ultimate upper bound $c^*=O(K^{-\gamma}+K \left[{\sqrt{\log K}}/{K}\right]^{4})$ and $\eta^*=O(K^{-\gamma}+K \left[{\sqrt{\log K}}/{K}\right]^{4})$.
  \label{lem:constantsExist}
 \end{proposition}

We make several remarks here. First, specifically, for balanced model classes where $\pi_{\max}/\pi_{\min}=1$, we conclude that the minimal separation condition is $c>c_0\eta_0=O(\sqrt{\log K})$. With the prior knowledge of ratios between standard deviations, our result matches the sharp separation threshold $O(\sqrt{\log K})$ established in \citep{Regev2017} for learning well-separated mixture of Gaussians. 
Second, \citep{Regev2017} also shows that for sufficiently large $C$ and $K > C^8$, there exist two mixtures $P,P'$ in the model class $\model(K,1/K,C^{-24}\sqrt{\log K})$ with unit variance for every components, such that their parameter distance is at least $C^{-24}\sqrt{\log K}$ but $d_{TV}(P,P')\leq K^{-C}$. However, with larger multipliers of $\sqrt{\log K}$ in the separation constants, our result further shows that the parameter distance is not diverging as $K\rightarrow \infty$. That is, the mixtures can be identified componentwisely no matter how many clusters there are under our separation conditions. 

Third, when $K$ is fixed, the separation lower bound $c_0\eta_0$ decreases in $d$. Informally, with the same $K$, $\pi_{\min},\pi_{\max}$, it is easier to distinguish two Gaussians  mixtures in higher dimension. This is an instance of the {\em blessing of dimensionality} for Gaussian Mixture Models \citep{Anderson2014TheMT}.

Fourth, constraints on maximal and minimal standard deviation of each components also provide a valid way of regularization. If the components of $P,P'$ satisfiy the relation $c > \rho_\sigma c_0$ where $\rho_\sigma = \max_{i,j}\{\sigma_i,\sigma_j'\}/\min_{i,j}\{\sigma_i,\sigma_j'\}\leq \eta_0$, then the one to one correspondance also holds. 
%
Finally, when \theoremref{thm:main1} applies, the two mixtures must have same number of components. \mmpmmp{move this to supplement?}This further leads to the following corollary showing that a well-separated Gaussian mixtures cannot be close in total variation distance to a single spherical Gaussian. In the proof in \sectionref{sec:proof-main}, we prove this corollary and specify the constants.
 \begin{corollary}
  For any $\epsilon <\pi_{\min}$, there is a positive constant $c_{\rm single}$ depending on $\pi_{\min}$ and $\epsilon$ such that when $c>c_{\rm single}$, no spherical Gaussians are within total variation distance $2\epsilon$ from any $P$ in  $P\in\mathcal{M}(K,\pi_{\min},c)$.
  \label{cor:cover-by-single}
 \end{corollary}
 We shall point out that the upper bounds in \theoremref{thm:main1} do not tends to 0 as $d_{TV}(P,P') \rightarrow 0$. Ideally, for a fixed mixture distribution $P\in\mathcal{M}(K,\pi_{\min},c)$, it holds that $\lim\inf TV(P,P')/d_{param}(P,P') > 0$ as $d_{param}(P,P')\rightarrow 0$, due to a local Taylor expansion analysis proposed in \citep{HoNguyen2016}. On the other hand, our ultimate bound $c^*$ and $\eta^*$ have two parts: the first part coming from the total variation distance between two Gaussians and the second part coming from not-far-enough separation between components. According to the author's knowledge it is not known if one can obtain computable parameter distance bound between $P,P'$ which scales as $O(TV(P,P'))$ as the total variation distance $TV(P,P')$ tends to zero.

\mmp{Special case $\sigma$ fixed for all mixtures in $\model$}
\mmp{Bound on $\pi$'s can be $>\pimin$, right? However, when thm applies, it guarantees a 1/1 correspondence between components, with closeness for $\mu,\sigma$ even if the bound for $\pi_i$ can become uninformative.}
\mmp{Prove or look up dparam small implies d clustering small--for later}

\subsection{\theoremref{thm:main1} Proof Sketch}
The detailed proofs of all theorems and lemmas mentioned here are in the Supplement. There is an initial stage in the proof, where bounds $c_0,\eta_0$ are obtained,  followed by a refinement stage.

\underline{Initial stage} 
First, we show in \theoremref{thm:step1-max-i-j} that if $\pi_{\min}$ is large and
$\epsilon$ is small, then for each component $P_j'$ of $P'$, there
exists a ``matched'' component $P_i$ of $P$ such that $ \norm{\mu_j'-\mu_i}\leq
c_0\max\{\sigma_j',\sigma_i\} $, where $c_0$ is given by \eqref{eq:c0}.

Next in \theoremref{thm:unique-matching}, we show that if separation is sufficiently large ($c>\eta_0c_0$), for each component $P_j'$ of $P'$, the matched component $P_i$\comment{that satisfies $\norm{\mu_i-\mu_j'}\leq c_0\max\{\sigma_j',\sigma_i\}$}is unique. This implies that $K=K'$. Relabeling the components of $P_j'$ we obtain the initial bounds $(c_0,\eta_0)$ such that $\norm{\mu_i-\mu_i'}\leq c_0\max\{\sigma_i',\sigma_i\}$ and $\max\{\sigma_i/\sigma_i',\sigma_i'/\sigma_i\}\leq \eta_0$ hold for all $i$.

\underline{Refinement} We show that if $\norm{\mu_i-\mu_i'}\leq c_b\max\{\sigma_i',\sigma_i\}$ and $\max\{\sigma_i/\sigma_i',\sigma_i'/\sigma_i\}\leq \eta_b$ hold for some $(c_b,\eta_b)$ and for all corresponding pairs in two \sphgmm~from $\mathcal{M}(K,\pi_{\min},\epsilon)$, the separation lower bound of $\mu_i,\mu'_j$ for $j\neq i$ (i.e. with the components of $P'$ that are not matched with $P_i$) can be improved. With this, \lemmaref{lem:component-wise-tv-P} implies a new upper bound $UB(c_b,\eta_b)$ on the total variation distance between each pair $TV(P_i,P_i')$ of matched Gaussian components, where
\begin{equation}
  UB(c_b,\eta_b):=\frac{2\epsilon}{\pi_{\min}} + \frac{2(1-\pi_{\min})}{\pi_{\min}}\Phi\left(-(c+\frac{c}{\eta_b}-c_b)/2\right)
\end{equation} 
is defined for all $c_b\in [0,c_0]$ and $\eta_b\in[1,\eta_0]$. 

Therefore from $(c_0,\eta_0)$ we obtain a new pair of upper bound constants $(c_1,\eta_1)$ from \lemmaref{lem:rho1-rho2}, with $c_1<c_0,\eta_1 < \eta_0$. They let us upper bound the total variation distance $TV(P_i,P_i')$ by $UB(c_1,b_1) < UB(c_0,b_0)$, which produces tighter constants $(c_2,b_2)$, and so on. 

The sequence of upper bounds $(c_t,\eta_t)$, where both $\{c_t\}$ and $\{\eta_t\}$ are positive and decreasing in $t$, converges to $\eta^*,c^*$. These limits are in equation \eqref{eq:cstar}, and the proof is in Supplement \ref{sec:refined}. Together with the component-wise comparison of proportions \lemmaref{lem:diff-proportion}, the proof of \theoremref{thm:main1} can now be completed.

\section{Numerical Examples}
\label{sec:num}
The upper bounds as well as the conditions in \theoremref{thm:main1} are computable and here we evaluate them on a variety of examples. 

\paragraph{Minimum Separation}  Recall that for \theoremref{thm:main1} to apply, both Gaussian mixtures $P,P'$ need to have well separated components, with relative separation $c\geq c_0\eta_0$. Here we calculate the minimal separation $c_0\eta_0$ in the limit case $\epsilon=0$. This will give a view of the domain of applicability of the theorem. 

We consider mixtures of spherical Gaussians in $d=5,20,35$ dimensional
Euclidean spaces with the number of components $K$ from 2 to 40. For
each $d$ and $K$, the ratio $\eta_\pi = \pi_{\max}/\pi_{\min}$ is set
to 1,2,4,8, and 16. From $K$ and $\eta_{\pi}$, we set
$\pi_{\min},\pimax$ so that $\pimax\;=\;1-(K-1)\pimin$.
%
%
\figureref{fig:sepVSK} illustrates what the minimal separation
requirements are so that \theoremref{thm:main1} can be applied. We observe that the heterogeneity of
mixture proportions $\eta_\pi$ indeed has a large effect on the separation
requirement; with $\eta_{\pi}$ as large as 16, one may need separation
constant $c\geq 12$ to apply \theoremref{thm:main1} in dimension 5
even with 2 clusters. On the other hand, for balanced mixtures, the
requirement is not so severe. For mixtures in $\mathcal{M}(2,1/2,c)$
namely with two equal components, even for $d=1$ the lower bound for $c$ is 2.29. As expected \citep{Anderson2014TheMT}, when $d$ increases, this requirement decreases as low as $c=1.55$  when $d=35$. When $K$ increases, $c_0\eta_0$ increases at the rate
of approximatevely $\sqrt{\log K}$. 
\begin{figure}[!t]
 \centering
 \includegraphics[width = 0.9\textwidth]{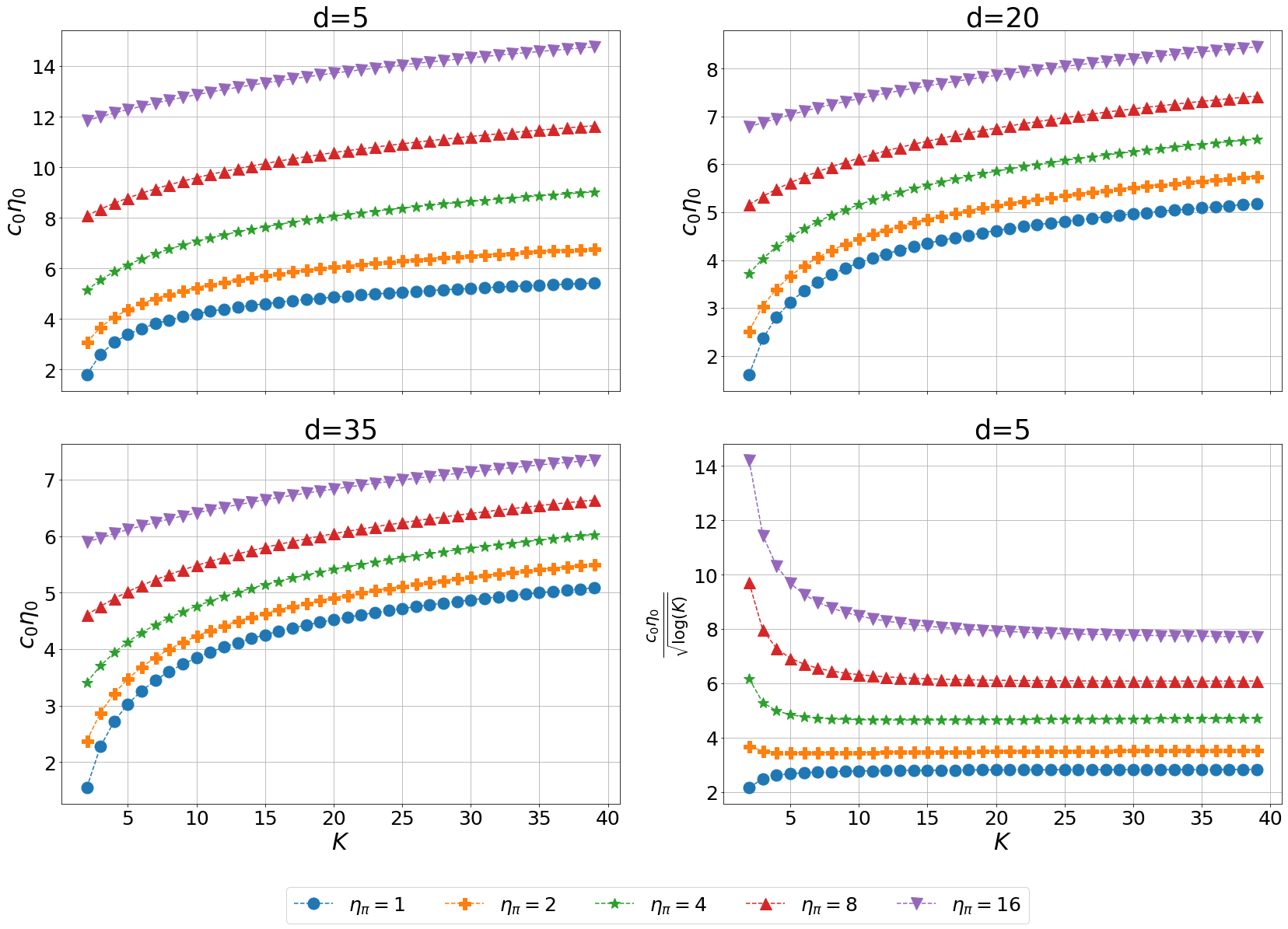}
 \label{fig:sepVSK}
 \caption{\textbf{Sufficient minimal separation} $c_0\eta_0$ in \theoremref{thm:main1}  under different settings. \textbf{Top left, Top right, Bottom Left} show the dependence of $c_0\eta_0$ on $K$ and $\eta_\pi=\pimax/\pimin$ in dimensions $d=5,20,35$, respectively. \textbf{Bottom right} shows that the dependence of $c_0\eta_0$ on $K$ asymptotes to $\sqrt{\log K}$.}
\end{figure}

\paragraph{Upper Bounds for parameter divergence} Here we present the numerical values of the upper bounds $c^*,\eta^*$ as well as of the upper bound
on the relative difference $|\pi_i-\pi'_{\myperm(i)}|/\pimin$ from
equation \eqref{eq:pibound}, for different levels of model fit
$\epsilon$.  We consider $P\in\mathcal{M}(K,1/(K+1),c)$  on $\mathbb{R}^{20}$, and vary
$K=2,5,10$, and $c=3,4,5,6$. \figureref{fig:iterativeUB} shows the
values of the three bounds in this scenario.  For
example, a mixture of $K=2$ components in $d=20$ dimensions requires a
separation $c\approx 3$ for $\epsilon\approx 0.01$ by \theoremref{thm:main1}. Once this condition holds, we have good guarantees: for each
pair of corresponding components,
\begin{equation}
 \norm{\mu_i-\mu_i'}\leq 0.151\max\{\sigma_i,\sigma_i'\},\quad \frac{\max\{\sigma_i,\sigma_i'\}}{\min\{\sigma_i,\sigma_i'\}} \leq 1.035,\quad |\pi_i-\pi_i'|\leq 0.02 \approx 0.06 \pi_{\min}\;.
\end{equation}
From \figureref{fig:iterativeUB} we see that when $\epsilon$ is small, all bounds are dominated by the separation $c$.  \mmp{Don't understand this: since \theoremref{thm:main} requires large for the separation condition is indeed large} As $\epsilon$ increases, we observe that all three bounds are dominated by $\epsilon$. Note the relation between these graphs and the orange curve $\eta_\pi=2$ from \figureref{fig:sepVSK}. \mmp{shall we expound?}
\begin{figure}[!t]
 \centering
 \includegraphics[width = \textwidth]{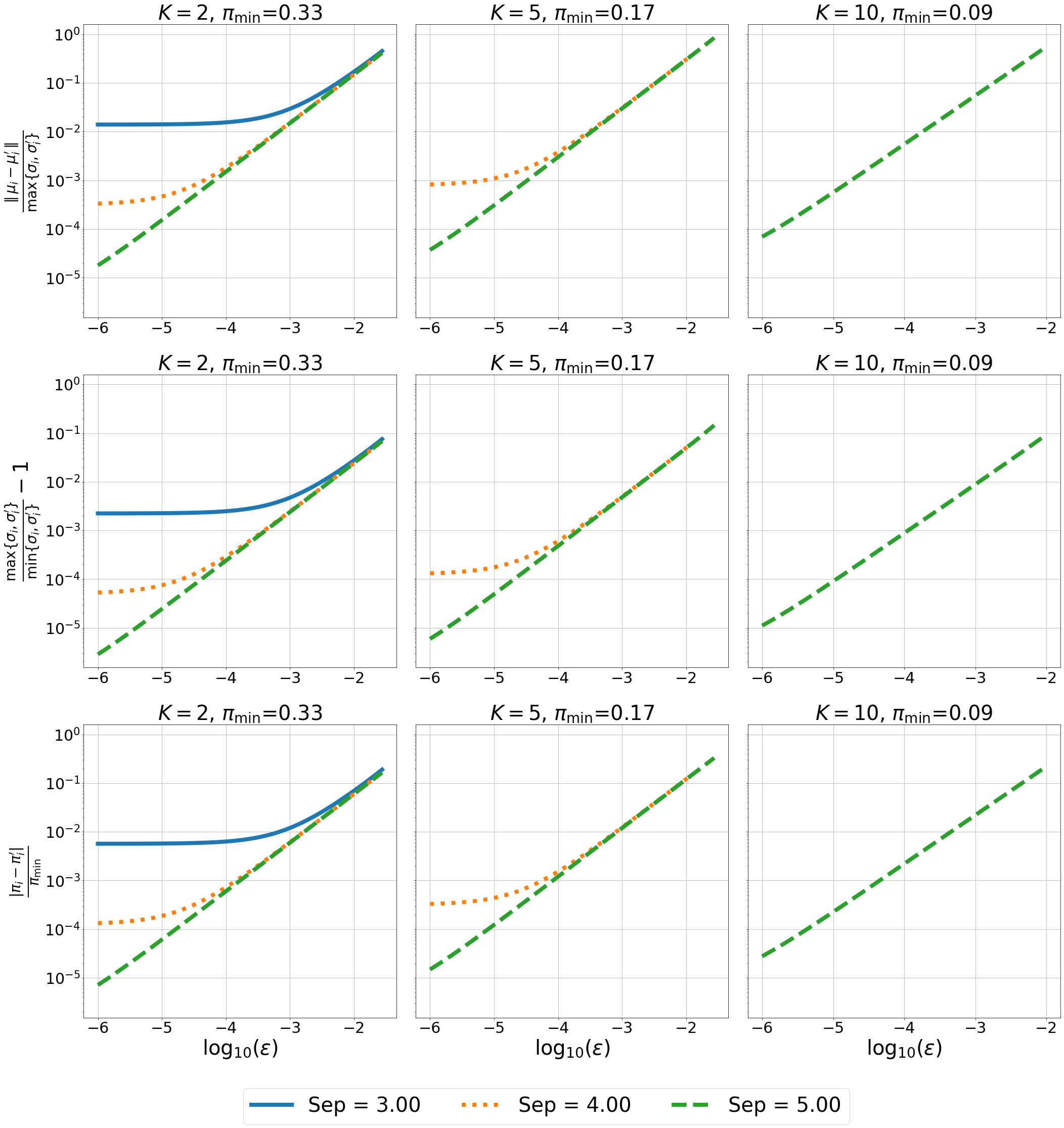}
 \caption{ \label{fig:iterativeUB}
   \textbf{Upper Bounds} on the distance between corresponding means $c^*$ (top row), ratio of standard deviations $\eta^*$ minus one (middle row) and the difference in mixture proportions measured in multiples of $\pi_{\min}$ (bottom row) from  \theoremref{thm:main1} for different values of $K$, and of separation  Sep$=c$, in $d=20$ dimensions. Some curves don't appear because the separation is not large enough to obtain the upper bound, as indicated by the Top, Left panel of Figure \ref{fig:sepVSK} (orange curve).\comment{ y-Axes of each row and x-axes of each column are shared. The first row shows how the centers differs under the units of maximal standard deviation; the second row shows the ratio of standard deviations minus one, and the third row shows the difference in proportions,. Since in our setting $\pi_{\max}$ can be as large as $2\pi_{\min}$, all of these difference measured in $\pi_{\min}$ are less than one, which shows that we also obtain interesting bounds on the difference of proportions.}}
\end{figure}

\section{Discussion}
\label{sec:related}

\paragraph{Comparisons with recent robust identifiability results}
\theoremref{thm:main1} is the first result to our knowledge to provide computable \emph{global} parameter stability upper bounds of mixtures of well-separated \sphgmm s, \mmpmmp{which filter out all parameters such that the total variation distance is large.} This result can also be seen as a {\em robust identifiability} result for $\model(K,\pimin,c)$.

When $TV(P,P')\rightarrow 0$, some identifiability result exists using $W_1$ distance of parameters: \citep{HoNguyen2016},\citep{Heinrich2018}, etc, show that for two sets of parameter distributions $G,G'$ with $W_1(G,G')$ small, $TV(P_G,P'_{G'})$ is  asymptotically greater than the $W_1(G,G')$ distance (without assuming separation). This result is almost the converse of our \theoremref{thm:main1}. 

The closest related works are the parameter identifiability theorems in robust learning of Gaussian mixtures, namely Theorem B.1 in \citep{Diakonikolas2017}, Theorem 8.1 in \citep{Allen2021}, Theorem 9.1 in \citep{ABakshi2020}, which are based on newly developed methods of moments proof techniques.

Consider two mixtures of general Gaussians $P,P'$ with maximal number of components being $K$. The assumptions made are that $TV(P,P')\leq\epsilon$, minimal weights of $P,P'$ are greater than $\epsilon^{b_1}$ and different components of $P,P'$ have separation in total variation distance greater than $\epsilon^{b_2}$, for sufficiently small constants $b_1,b_2>0$ depending only on $K$. The key observation used in the proofs is that the parameters distances of $P_k,P'_{\myperm( k)}$ are bounded by $poly(\epsilon)$ if the Hermite moment polynomials is bounded by $poly(\epsilon)$, which is guaranteed by the assumption that the total variation distance between $P,P'$ is upper bounded by $\epsilon$.
Their proof technique can be directly modified into proving that $P$ is $(\epsilon,\tilde{\delta}(\epsilon))$ stable with $\tilde{\delta}(\epsilon)=O_k(\epsilon^{b_3^K})$, where $b_3$ is another sufficiently small positive constant. 

When $P$ is fixed (the case of interest both for us and for robust identifiability in general), if $\epsilon \rightarrow 0$ these results are tighter than ours as the bound $\tilde{\delta}\rightarrow 0$. In this case, the assumptions on $P$ become more relaxed than ours. \comment{ and minimal weights requirements on $P$ and $P'$ than ours.}However, it is not known how to determine the unspecified constants in $\tilde{\delta}$ from their proof techniques, and the dependence on $K$ and exponent of $\epsilon$ are not optimal (\cite{Regev2017}, respectively in the discussion in Section \ref{sec:setup}).

For larger, more realistic $\epsilon$, the rate $\tilde{\delta}(\epsilon)\sim \epsilon^{b_3^K}$ is much slower than what our numerical simulations achieve, which appears as $\delta(\epsilon,\model)\sim \epsilon$. In  this case, the assumptions on $\pimin$ and separation of \theoremref{thm:main1} remain fixed, while the assumptions in \cite{Allen2021,ABakshi2020} become more restrictive (or may even not apply) with larger $\epsilon$.  Hence, our results are more useful for robust recovery, where extending stability to larger $\epsilon$ is desired, while \citep{Allen2021,ABakshi2020} are useful in the limit $\epsilon\rightarrow 0$ and for {\em algorithmic} stability (the main goal of these works). This growth of  $\delta(\epsilon,\model)\sim \epsilon$ matches the sharp threshold of \citep{HoNguyen2016}, which is obtained asymptotically for $\epsilon\rightarrow 0$, suggesting that \citep{HoNguyen2016} may hold for larger $\epsilon$ and that our worst case bound may match it at least under certain conditions. 

Regarding the dependence on $K$, a sharp threshold for \sphgmm~was obtained by \citep{Regev2017}, who show that $\Omega(\sqrt{\log K})$ separation is necessary and sufficient to recover this class of GMM in polynomial sample size and time. Our \theoremref{thm:main1} is approximately matching this threshold. 

\paragraph{Results on other related problems}

Solving for exact MLE in Spherical Gaussian Mixture Models has been proved to be
NP-hard \citep{ToshDasgupta2018}. The current mainstream approach to
fitting GMM is the \emph{Expectation-Maximization} (EM) algorithm
\citep{EM1977};\comment{ which iteratively increases the value of the
log-likelihood. \citep{BinEM2017} prove the consistency of
population-level first-order EM when it is initialized within a
neighborhood of the true parameter. } One cannot
prove that EM will converge to the global maximum of log-likelihood
function without further assumptions, and in fact,  EM
can converge to bad local maxima with high probability
\citep{Chi2016}, even for well-separated Gaussians; the same paper
also confirms the existence of local maxima in the population
likelihood function. 
This negative result makes it the more necessary
to have a post-processing validation stage, in which to be able to
reject bad local optima. The results in our paper can fulfill this
role, in the population sense.

A second approach to estimating GMM parameters with consistency guarantees relies on \emph{Methods
of Moments} (MOM) \citep{Pearson1894}, and is exemplified in
\citep{mom-shamkakade,Yihong2020}. \citep{mom-shamkakade} used moments up to the third order to learn an \sphgmm. They don't require separation conditions, but need the means of different components to be linear independent. Therefore, this method cannot be applied to low $d$ and large $K$ scenario. \citep{Yihong2020} used an semi-definite programming based denoising procedure of moments up to order $O(K)$. \mmpmmp{The guarantees usually cannot be applied in large $K$,low $d$ scenario. Require $K\leq d$??}

When the cluster components are well-separated, a different line of
research proposed in the seminal paper \citep{Dasgupta1999} and
refined in \citep{Arora2001,Achlioptas2005,Kannan2008},etc, and \citep{DasguptaSchulman2007}, provides an accurate estimation of GMM parameters with high
probability. All these results are predicated on the data being
sampled from a mixture of well separated Gaussians.\mmpmmp{from the model class}

Also, previously very few results have been established on evaluating the result of a Gaussian Mixture fit. In these works, e.g. \citep{drtonPlummer:2017}, \citep{Huang2017}), the main target is model selection, especially consistently estimating the number of components in a Gaussian Mixture population.

\comment{
One can also use heuristic scores including rand index \citep{Rand1971}, silhouette coefficients \citep{Rousseeuw1987}, normalized mutual information \citep{Vinh2010} etc and their variants to evaluate the result of the associated soft clustering. These methods are heuristic and do not have a theoretical guarantee. Further, they also cannot answer the question of whether this fitted model (or model class) is meaningful to this particular dataset.}
\comment{{\bf Model free guarantees for unsupervised learning}
Previous model free validation results have been successfully established on non-probabilistic clustering problems, such as  loss-based hard clustering \citep{Meila2018}. The only model free guarantees in a model-based clustering framework that we are aware of are obtained by \cite{Yali2016} which focus on graph partitioning under a general family of network block models\comment{ \citep{MWan:nips15}}, including the Stochastic Block Model \citep{holland1983stochastic}, and the Degree Corrected Stochatic Block Model \citep{qinRohe:13}. The current paper does not focus on the validation of clustering but on the parameter learning problem.}
\comment{
In \citep{HM2016} a different framework of stability guarantees for general unsupervised learning is proposed. Similar to our proposal, they also need a hypothesis class and the assumption that the hypothesis class fits the data well. However, the guarantee of \citep{HM2016} are with respect to performance measured by a loss function, instead of the latent structure or parameters themselves. For instance, under their framework it is possible to obtain stability guarantees for non-identifiable models.
}

{\bf Bootstrap stability} A series of papers
\citep{lange:04,Rakhlin2006,Shai2007-kmeans,BenDavidBoundary,OhadTishby2009,Shamir2010-MLJournal}
discuss the properties of "bootstrap stability" of K-means and
Gaussian Mixtures, focusing on the stability of a clustering under
resampling of the data, or perturbations in the clustering
algorithm. The goal in these work was to use bootstrap stability
in selecting the number of clusters $K$.
\comment{But they do not provide a guarantee on a particular dataset and a know
estimation result. The stability proposed in these works is measuring
different uncertainty compared to ours. Their uncertainty mostly comes
from sampling errors. In the Model Free Validation framework used in
this paper, stability is a concept associated with both the dataset
and the model class. It provides a tool to measure and evaluate
intractable problems quantitatively.}


In conclusion, this paper obtains the first {\em computable} robust identifiability bounds for spherical Gaussian mixtures, in the population setting. Our bounds can be extended to  Gaussian mixtures with full covariance matrices $\Sigma_k$, and bounded excentricity. The bounds $\delta(\epsilon,\model)$ match the known sharp threshold w.r.t $K$ from \cite{Regev2017}, and are uniform over the model class $\model(K,\pimin,c)$. 

In the proof we introduce an iterative approach for tightening
the  bounds which is original, to our knowledge. Several
other results of our Lemmas can be of independent interest, such as a tighter bounds
on parameter variation for a single Gaussian, that remain informative
for larger perturbations in $TV$ distance than the previous bound of \cite{devroye2020total}.

\section*{Acknowledgment}

The author acknowledges support from NSF DMS award 1810975.

\bibliographystyle{plainnat}
\bibliography{mixture}


\newpage
\appendix
\section{Proof of Main Results}
\label{sec:proof-main}

\subsection{Proof of Proposition \ref{lem:constantsExist}}
\begin{proof}[Proof of Proposition \ref{lem:constantsExist}]
    First $c_0$ exists and is unique from \eqref{eq:c0}. The definition of 
    $c_0$ shows that 
    \begin{equation}
    \Phi(-c_0/2)=1-\Phi(c_0/2) = 1-\Phi(\Phi^{-1}(1-\frac{\pi_{\min}-2\epsilon}{2}))=\frac{1}{2}(\pi_{\min}-2\epsilon)\;.
    \end{equation}
    
    Consider $R(\eta)$ to be a function defined by the the R.H.S. of \eqref{eq:eta0-diff} tending to zero, i.e.
    \begin{equation}
        R(\eta) = F_d\left(\frac{2\eta^2\log \eta}{\eta^2-1} \right) - F_d\left(\frac{2\log\eta}{\eta^2-1}\right)
    \end{equation}
   and $R(\eta)\rightarrow 0$ when $\eta\rightarrow 1^+$.
   According to \lemmaref{lem:mono-eta-function}, $R(\eta)$ is increasing in $\eta$ and tends to one when $\eta\rightarrow \infty$. Denote 
    \begin{align}
    L(\eta) &= 1- \frac{\pi_{\min}-2\epsilon}{\pi} + \frac{1-\pi}{\pi}2\Phi(-\frac{\eta_0c_0}{2}) \;.
    \end{align}
    Then the L.H.S. of \eqref{eq:eta0-diff} can both be written as $L(\eta)$. Since
    \begin{align}
   L(1) &= 1- \frac{\pi_{\min}-2\epsilon}{\pi} + \frac{1-\pi}{\pi}(\pi_{\min}-2\epsilon) \\
    &= 1-(\pi_{\min}-2\epsilon)\left(-\frac{1}{\pi}+\frac{1}{\pi}-1\right) \\
    &= 1-\pi_{\min}+2\epsilon > 0
    \end{align}
    and when $\eta\rightarrow \infty$, $L(\eta)$ tends to $1-(\pi_{\min}-2\epsilon)/\pi <1$. Therefore since $L(\eta)$ is decreasing in $\eta$ and $R(\eta)$ is increasing, there exists a unique $\eta_0$ that satisfies equation \eqref{eq:eta0-diff}. 
   
    For the second part of this lemma, we now consider the case $\pi_{\min}=\Omega(K^{-(1+\alpha)})$ for some $\alpha \geq 0$,$\pi_{\max}/\pi_{\min}=O(K^\beta)$ for some $0\leq \beta \leq 1+\alpha$ and $\epsilon$ small enough such that $\pi_{\min}-2\epsilon=\Omega(K^{-(1+\alpha)})$. By \eqref{eq:c0}, it holds that 
    \begin{equation}
     \Phi(-\frac{1}{2}c_0) = \Omega\left(\frac{1}{2K^{1+\alpha}}\right)\;.
    \end{equation}
    Note that for any constant $\zeta>0$,
    \begin{equation}
    \Phi(-\frac{1}{2}2\sqrt{2\log(2K^{1+\alpha}/\zeta)})\leq \exp\left(-\frac{1}{8}(2\sqrt{2\log(2K^{1+\alpha}/\zeta)})^2\right) = \frac{\zeta}{2K^{1+\alpha}} \lesssim \Phi(-\frac{1}{2}c_0)
    \end{equation}
    where $a \lesssim b$ means there exists a constant $A>0$ such that $a \leq A\cdot b$. Hence there exists some $\zeta_0>0$ such that $$c_0 \leq 2\sqrt{2\log(\frac{2K^{1+\alpha}}{\zeta_0})} = O(\sqrt{\log K})$$. 
    
    Now we upper bound $\eta_0$. Note that for any $t > 0$, let 
    \begin{equation}
        g(t)=1-\Phi(t)-\frac{1}{\sqrt{2\pi}}\frac{t}{t^2+1}\exp(-\frac{1}{2}t^2).
    \end{equation}
    Since $g'(t)=-2e^{-t^2/2}/(t^2+1)^2$, $g(t)$ is strictly decreasing. Since $\lim_{t\rightarrow \infty}g(t) = 0$, we have  $g(t)\geq 0$, i.e., 
    \begin{equation}
       1-\Phi(t)\geq \frac{1}{\sqrt{2\pi}}\frac{t}{t^2+1}\exp(-\frac{1}{2}t^2) > \frac{1}{5\sqrt{2\pi} t}\exp(-\frac{1}{2}t^2)
    \end{equation}
    holds for all $t\geq 1/2$. Therefore since $\Phi(-c_0/2)\leq 1/2K\leq 1/4$ or $c_0/2 \geq 0.67 > 0.5$
    \begin{equation}
        \frac{1}{2K}\geq \Phi(-\frac{1}{2}c_0)\geq \frac{1}{5\sqrt{2\pi} \frac{1}{2}c_0}\exp(-\frac{1}{8}c_0^2)\;,
    \end{equation}
    we conclude that $\exp(-\frac{1}{8}c_0^2)= O(\sqrt{\log K}/K)$. Therefore, 
    \begin{align}
        L(\eta) &= 1-\frac{\pi_{\min}-2\epsilon}{\pi_{\max}}+\frac{1-\pi_{\max}}{\pi_{\max}}2\Phi(-\frac{1}{2}\eta c_0) \\
        &\leq 1-\frac{\pi_{\min}-2\epsilon}{\pi_{\max}}+ 2K\exp\left(-\frac{1}{2}\eta^2 c_0^2\right) \\
        &= 1-\frac{\pi_{\min}-2\epsilon}{\pi_{\max}}+ O\left(K \left[\frac{\sqrt{\log K}}{K}\right]^{4\eta^2}\right)\;.
    \end{align}
   
    On the other hand the R.H.S. of \eqref{eq:eta0-diff} is lower bounded by 
    \begin{equation}
    R(\eta) > 1 - \left(\frac{2\eta}{\eta^2+1}\right)^{\frac{d}{2}} > 1 - \left(\frac{2}{\eta}\right)^{\frac{d}{2}}.
    \end{equation}
   Hence take $\tilde{\eta}=2(2\pi_{\max}/(\pi_{\min}-2\epsilon))^{2/d}=O(K^{2\beta/d})$, it holds that
   \begin{equation}
       R(\tilde{\eta})-L(\tilde{\eta}) > \frac{1}{2}\frac{\pi_{\min}-2\epsilon}{\pi_{\max}}-\Omega\left(K \left[\frac{\sqrt{\log K}}{K}\right]^{4\tilde{\eta}^2}\right) \;.
   \end{equation}
    However by the selection of $\tilde{\eta}$, $R(\tilde{\eta})-L(\tilde{\eta})$ has to be positive for sufficiently large $K$. Therefore $\eta_0 < \tilde{\eta}$ and we conclude that for sufficiently large $K$ $\eta_0=O(K^{2\beta/d})$.
   
    Finally we will consider the estimates of $c^*$ and $\eta^*$. To start with, the R.H.S. of \eqref{eq:cstar} is always upper bounded by
   
    \begin{align}
       \frac{2\epsilon}{\pi_{\min}}+\frac{2(1-\pi_{\min})}{\pi_{\min}}\Phi(-\frac{1}{2}[c(1+\frac{1}{\eta^*})-c^*])&\leq \gamma + 2K\exp\left(-\frac{1}{2}c_0^2\eta_0^2\right) \\
       &= \frac{1}{K^\gamma} + O\left(K\left[\frac{\sqrt{\log K}}{K}\right]^{4\eta_0^2}\right)\;.
   \end{align}
   The R.H.S. of \eqref{eq:cstar} and \eqref{eq:etastar} is $O(1)$ and smaller than 1 for sufficiently large $K$. Note that for two spherical Gaussians $P_1=N_d(\mu_1,\sigma_1^2I_d),P_2=N_d(\mu_2,\sigma_2^2 I_d)$, from \lemmaref{lem:GaussianTV} and conclude the bounds on $c^*$ and $\eta^*$.
   \end{proof}

\subsection{Technical Tools for Proving \theoremref{thm:main1}}
Instead of directly diving into the proof of \theoremref{thm:main1}, we start by some observations in mixture models and analysis on total variation distances between spherical Gaussians. It is worthwhile noticing that in the lemmas in this subsection, we don't require the two mixture models $P,P'$ to have the same number of components.

\subsubsection{Total Variation Distance Between Spherical Gaussians}
We develop a lemma upper bounding parameter distances between
spherical Gaussians given their total variation distance. For this, one can use
Hellinger distance, which includes an anlytical expression as a
natural lower bound \citet{Gibbs2002probDist}, but they are usually not
tight in constants and therefore they cannot separate the bound in
mean parameter and variance parameter. \citet{devroye2020total}
proposes another lower bound starting from the definitions but their results are
only meaningful when the total variation distance is sufficiently
small.
\begin{lemma}
  \label{lem:GaussianTV}
  Suppose $P_1 = N_d(\mu_1,\sigma_1^2 I_d)$ and $P_2 = N_d(\mu_2,\sigma_2^2 I_d)$. Let $C_0(\rho)=2\Phi^{-1}(1-\frac{\rho}{2})$ and $\eta_0(\rho)$ be the solution of
   \begin{equation}
    1-\rho = F_d\left(\frac{2\eta^2 \log \eta}{\eta^2-1}\right) - F_d\left(\frac{2 \log \eta}{\eta^2-1}\right)\;.
  \end{equation}
  Then the following holds:
  \begin{itemize}
    \item If $\norm{\mu_1-\mu_2}\geq C_0(\rho)(\sigma_1+\sigma_2)/2$, then $TV(P_1,P_2)\geq 1-\rho$. Equality holds iff $\sigma_1 = \sigma_2$ and $\norm{\mu_1-\mu_2} = C_0(\rho)(\sigma_1+\sigma_2)/2$.
    \item If $\max\{\sigma_1/\sigma_2,\sigma_2/\sigma_1\}\geq \eta_0(\rho)$, then $TV(P_1,P_2)\geq 1-\rho$. Equality holds iff $\mu_1 = \mu_2$ and $\max\{\sigma_1/\sigma_2,\sigma_2/\sigma_1\} = \eta_0(\rho)$.
  \end{itemize}
\end{lemma}
We separate the Gaussian total variation lower bound into three different lemmas and prove them.
\begin{lemma}
    Suppose $P_1 = N_d(\mu_1,\sigma_1^2I_d)$ and $P_2= N_d(\mu_2,\sigma_2^2 I_d)$. Let $\Phi$ be the CDF of standard normal distribution and define for each $0 < \rho < 1$, 
    \begin{equation}
    C_0(\rho) = 2\Phi^{-1}(1-\frac{\rho}{2})
    \end{equation}
    If $C=\norm{\mu_1-\mu_2}/\max\{\sigma_1,\sigma_2\} \geq C_0(\rho)$, then there exists a set $A$ such that $P_1(A)-P_2(A) \geq 1-\rho$. Equality holds iff $\sigma_1 = \sigma_2$ and $C=C_0(\rho)$.
    \label{lem:tv-mono}
   \end{lemma}

\begin{proof}[Proof of \lemmaref{lem:tv-mono}] Let $p_1,p_2$ be the density of $P_1,P_2$ and random variables $X_1,X_2$ follow $P_1,P_2$ respectively.\\

 WLOG assume that $\sigma_1 \leq \sigma_2$. If the statement holds under this case, then when $\sigma_1 > \sigma_2$, one can select $\overline{A}$ such that $P_2(\overline{A})-P_1(\overline{A}) > 1-\rho$. Then take $A=\overline{A}^\complement$ and the desired result holds.
 
 To start with, we consider the case $d=1$ and assume $\mu_2-\mu_1 = C\sigma_2=C\eta \sigma_1>0$. When $\sigma_2 > \sigma_1$, let $A$ be the set
 $\{x\in\mathbb{R}: x_{0,-} \leq x\leq x_{0,+} \}$
 where
 \begin{align}
 x_{0,\pm} &= \frac{\frac{\mu_1}{\sigma_1^2}-\frac{\mu_2}{\sigma_2^2} \pm \sqrt{\frac{(\mu_1-\mu_2)^2}{\sigma_1^2\sigma_2^2}+2(\frac{1}{\sigma_1^2}-\frac{1}{\sigma_2^2})\log \frac{\sigma_2}{\sigma_1}}}{\frac{1}{\sigma_1^2}-\frac{1}{\sigma_2^2}}
 \label{eq:x0pm}
 \end{align}
 $x_{0,\pm}$ are the two roots of the equation that the two normal densities equal, i.e. 
 \begin{equation}
 \frac{1}{\sqrt{2\pi}\sigma_1}\exp\{-\frac{(x-\mu_1)^2}{2\sigma_1^2}\}=\frac{1}{\sqrt{2\pi}\sigma_2}\exp\{-\frac{(x-\mu_2)^2}{2\sigma_2^2}\}
 \label{eq:prob-equal-mu}
 \end{equation}
 
 Denote $\eta=\sigma_2/\sigma_1 \geq 1$, when $\eta > 1$ it holds that
 \begin{align}
 x_{1,\pm} &= \frac{x_{0,\pm}-\mu_1}{\sigma_1} = \frac{-C\eta \pm \eta\sqrt{C^2\eta^2+2(\eta^2-1)\log \eta}}{\eta^2-1} \\
 x_{2,\pm} &= \frac{x_{0,\pm}-\mu_2}{\sigma_2} = \frac{-C\eta^2\pm \sqrt{C^2\eta^2+2(\eta^2-1)\log \eta} }{\eta^2-1}
 \end{align}
 Let $\varphi$ be the density function of standard normal distribution then $\varphi(x_{2,+})=\eta \varphi(x_{1,+}),\varphi(x_{2,-})=\eta \varphi(x_{1,-})$ by the selection of $x_{0,\pm}$. Since the desired expression 
 \begin{align}
 P_1(A)-P_2(A) &= Pr(x_{1,-}\leq \frac{x-\mu_1}{\sigma_1} \leq x_{1,+}) - Pr(x_{2,-}\leq \frac{x-\mu_2}{\sigma_2} \leq x_{2,+}) \\
 &= \Phi(x_{1,+}) - \Phi(x_{1,-}) - \Phi(x_{2,+}) + \Phi(x_{2,-})
 \end{align}
 is implicitly a function of $C,\eta$, we denote it as $h_+(C,\eta)$. Note that since
 \begin{align}
 x_{0,\pm}&=\frac{\left(\frac{\mu_1}{\sigma_1^2}-\frac{\mu_2}{\sigma_2^2}\right)^2 - \frac{(\mu_1-\mu_2)^2}{\sigma_1^2\sigma_2^2}-2(\frac{1}{\sigma_1^2}-\frac{1}{\sigma_2^2})\log \frac{\sigma_2}{\sigma_1}}{\left(\frac{1}{\sigma_1^2}-\frac{1}{\sigma_2^2}\right)\left(\frac{\mu_1}{\sigma_1^2}-\frac{\mu_2}{\sigma_2^2} \mp \sqrt{\frac{(\mu_1-\mu_2)^2}{\sigma_1^2\sigma_2^2}+2(\frac{1}{\sigma_1^2}-\frac{1}{\sigma_2^2})\log \frac{\sigma_2}{\sigma_1}}\right)}=\frac{\frac{\mu_1^2}{\sigma_1^2}-\frac{\mu_2^2}{\sigma_2^2}-2\log \frac{\sigma_2}{\sigma_1}}{\frac{\mu_1}{\sigma_1^2}-\frac{\mu_2}{\sigma_2^2} \mp \sqrt{\frac{(\mu_1-\mu_2)^2}{\sigma_1^2\sigma_2^2}+2(\frac{1}{\sigma_1^2}-\frac{1}{\sigma_2^2})\log \frac{\sigma_2}{\sigma_1}}}
 \end{align}
 Consider the limit $\sigma_2\rightarrow \sigma_1$ with $\sigma_1$ fixed and recall that $\mu_2-\mu_1 = C\sigma_2 > 0$, it holds that $\lim_{\sigma_2\rightarrow \sigma_1}x_{0,-}=-\infty$ and $\lim_{\sigma_2\rightarrow \sigma_1}x_{0,+}=(\mu_1+\mu_2)/2$. And further 
 \begin{align}
     \lim_{\sigma_2 \rightarrow \sigma_1^+} x_{1,+}&= \frac{\mu_2-\mu-1}{2\sigma_1} = \frac{C}{2},\quad \lim_{\sigma_2 \rightarrow \sigma_1^+} x_{1,-}= -\infty \\
     \lim_{\sigma_2 \rightarrow \sigma_1^+} x_{2,+}&= \frac{\mu_1-\mu_2}{2\sigma_2} = -\frac{C}{2},\quad \lim_{\sigma_2 \rightarrow \sigma_1^+} x_{2,-}= -\infty
 \end{align}
 which implies, since $\Phi$ is continuous,
 \begin{equation}
    \lim_{\eta\rightarrow 1^+} h_+(C,\eta) =1-2\Phi(-\frac{C}{2})
    \end{equation}
 
 Observe that $x_{1,\pm}-\eta x_{2,\pm} = C\eta$ or $x_{2,\pm} = x_{1,\pm}/\eta-C$. Then
 \begin{equation}
 \frac{\partial x_{1,\pm}}{\partial \eta}-\eta\frac{\partial x_{2,\pm}}{\partial \eta} = \frac{\partial x_{1,\pm}}{\partial \eta} - \eta\left[-\frac{1}{\eta^2}x_{1,\pm}+\frac{1}{\eta}\frac{\partial x_{1,\pm}}{\partial \eta}\right] = \frac{1}{\eta} x_{1,\pm}\;.
 \end{equation}
 
 \eqref{eq:prob-equal-mu} shows that $\eta\varphi(x_{1,\pm})=\varphi(x_{2,\pm})$ Then, by taking partial derivative with $\eta$, it holds that
 \begin{align}
 \frac{\partial h_+(C,\eta)}{\partial \eta} &= \varphi(x_{1,+})\frac{\partial x_{1,+}}{\partial \eta} - \varphi(x_{1,-})\frac{\partial x_{1,-}}{\partial \eta} -\varphi(x_{2,+})\frac{\partial x_{2,+}}{\partial \eta} + \varphi(x_{2,-})\frac{\partial x_{2,-}}{\partial \eta} \\
 &=\varphi(x_{1,+})\left[\frac{\partial x_{1,+}}{\partial \eta}-\eta \frac{\partial x_{2,+}}{\partial \eta}\right] - \varphi(x_{1,-})\left[\frac{\partial x_{1,-}}{\partial \eta}-\eta \frac{\partial x_{2,-}}{\partial \eta}\right] \\
 &=\frac{1}{\eta}(x_{1,+}\varphi(x_{1,+})-x_{1,-}\varphi(x_{1,-})) \\
 &>0\;.
 \label{eq:tv-mono-derivative}
 \end{align}
 where in the last inequality we use the observation that $x_{1,-} < 0 < x_{1,+},|x_{1,-}| > |x_{1,+}|$. Therefore $h_+(C,\eta)$ is increasing with $\eta$ when $C$ is fixed. And when $C>C_0(\rho)$ it holds that $h_+(C,\eta)> 1-2\Phi(-C_0(\rho)/2)=1-\rho$.
 
 When $\mu_1-\mu_2 = C\sigma_2 > 0$ again define
 \begin{align}
 x_{1,\pm} &= \frac{x_{0,\pm}-\mu_1}{\sigma_1} = \frac{C\eta \pm \eta\sqrt{C^2\eta^2+2(\eta^2-1)\log \eta}}{\eta^2-1} \\
 x_{2,\pm} &= \frac{x_{0,\pm}-\mu_2}{\sigma_2} = \frac{C\eta^2\pm \sqrt{C^2\eta^2+2(\eta^2-1)\log \eta} }{\eta^2-1}\;.
 \end{align}
 
 Consider $h_-(C,\eta) = P_1(A)-P_2(A)=\Phi(x_{1,+}) - \Phi(x_{1,-}) - \Phi(x_{2,+}) + \Phi(x_{2,-})$. First we observe that $\lim_{\eta\rightarrow 1^+} x_{1,+}=+\infty$ and $\lim_{\eta\rightarrow 1^+} x_{1,-}=-C/2$. This shows that 
 \begin{equation}
 \lim_{\eta\rightarrow 1^+} h_-(C,\eta) =1-2\Phi(-\frac{C}{2})
 \end{equation}
 still holds. Note that we still have $x_{2,\pm}=x_{1,\pm}/\eta+C$. Taking derivatives w.r.t. $\eta$ with same argument we have
 \begin{equation}
 \frac{\partial h_-(C,\eta)}{\partial \eta} = \frac{1}{\eta}(x_{1,+}\varphi(x_{1,+})-x_{1,-}\varphi(x_{1,-}))\;.
 \end{equation}
 Since $x_{1,-} < 0 < x_{1,+}$, the partial derivative is still positive. Therefore the result holds with similar arguments. Combining the two cases we complete the proof for the case $d=1$.
 
 Finally when $d\geq 2$, we consider the set of $x$ projected on to the one dimensional space spanned by $\mu_2-\mu_1$. 
 
 When $\sigma_2 = \sigma_1$, the set $A=\{x\in\mathbb{R}^d: p_1(x)\geq p_2(x)\}$ is the half space 
 $$
 A=\left\{x\in\mathbb{R}^d: \langle x, \frac{\mu_2-\mu_1}{
     \norm{\mu_2-\mu_1} \rangle \leq \frac{\mu_1+\mu_2}{2}
 }\right\}
 $$
 Therefore $P_1(A)-P_2(A)=1-2\Phi(-\frac{C}{2})\geq 1-\rho$ since $C\geq C_0(\rho)=2\Phi^{-1}(1-\rho/2)$. Equality holds iff $C=C_0(\rho)$.
 
 When $\sigma_2/\sigma_1>1$, consider $\tilde{P}_1 = N(\tilde{\mu}_1,\sigma_1^2)$ with $\tilde{\mu}_1=\langle \mu_1, \frac{\mu_2-\mu_1}{\norm{\mu_2-\mu_1}} \rangle$ and $\tilde{P}_2 = N(\tilde{\mu}_2,\sigma_2^2)$ with $\tilde{\mu}_2=\langle \mu_2, \frac{\mu_2-\mu_1}{\norm{\mu_2-\mu_1}} \rangle$. Compute $\tilde{x}_{0,\pm}$ by equation \eqref{eq:x0pm} with $\tilde{P}_1,\tilde{P}_2$ and let set
 
 \begin{equation}
 A =\left\{x\in \mathbb{R}^d:\tilde{x}_{0,-} \leq \langle x, \frac{\mu_2-\mu_1}{\norm{\mu_2-\mu_1}} \rangle \leq \tilde{x}_{0,+} \right\}\;,\quad \tilde{A} = \{x\in\mathbb{R}:\tilde{x}_{0,-}\leq x \leq \tilde{x}_{0,+}\}\;.
 \end{equation}
 Note that $|\tilde{\mu}_2-\tilde{\mu}_1| = \norm{\mu_2-\mu_1} = C\max\{\sigma_1,\sigma_2\}$ and thus 
 \begin{align}
 P_1(A)-P_2(A) = \tilde{P}_1(\tilde{A}) - \tilde{P}_2(\tilde{A}) > 1-\rho\;.
 \end{align}
 \end{proof}

\begin{lemma}
    Under the same setting of lemma \ref{lem:tv-mono}, if $C=\norm{\mu_1-\mu_2}/(\sigma_1+\sigma_2)\geq C_0(\rho)/2$, then there exists a set $A$ such that $P_1(A)-P_2(A) \geq 1-\rho$. Equality holds iff $\sigma_1 = \sigma_2$ and $C=C_0(\rho)/2$.
    \label{lem:etaplusone}
\end{lemma}

\begin{proof}[Proof of lemma \ref{lem:etaplusone}]
 The proof is mostly unchnaged compared with lemma \ref{lem:tv-mono}. We still consider the case $d=1$ and assume $\sigma_1\leq \sigma_2$ first. Once this is established, the remaining can be obtained from similar argument as in the proof of lemma \ref{lem:tv-mono}. Let $x_{0,\pm}$ be the same as in proof of lemma \ref{lem:tv-mono}.
 
 When $\mu_2-\mu_1 = C(\sigma_1+\sigma_2)>0$, now we should define
 \begin{align}
 x_{1,\pm} &= \frac{x_{0,\pm}-\mu_1}{\sigma_1} = \frac{-C(\eta+1) \pm \eta\sqrt{C^2(\eta+1)^2+2(\eta^2-1)\log \eta}}{\eta^2-1} \\
 x_{2,\pm} &= \frac{x_{0,\pm}-\mu_2}{\sigma_2} = \frac{-C\eta(\eta+1)\pm \sqrt{C^2(\eta+1)^2+2(\eta^2-1)\log \eta} }{\eta^2-1}\;,
 \end{align}
 Similarly to the proof of \lemmaref{lem:tv-mono}, it holds that
 \begin{equation}
 h_+(C,\eta) = \Phi(x_{1,+})-\Phi(x_{1,-}) - \Phi(x_{2,+})+\Phi(x_{2,-})\;,
 \end{equation}
 and 
 \begin{equation}
 \lim_{\eta\rightarrow 1^+} h_+(C,\eta) = 1-2\Phi(-C)\;.
 \end{equation}

 Since $x_{1,\pm}-\eta x_{2,\pm}=C(\eta+1)$ or $x_{2,\pm}=x_{1,\pm}/\eta-C(1+1/\eta)$, then 
 \begin{equation}
 \frac{\partial x_{1,\pm}}{\partial \eta}-\eta\frac{\partial x_{2,\pm}}{\partial \eta} = \frac{\partial x_{1,\pm}}{\partial \eta} - \eta\left[-\frac{1}{\eta^2}x_{1,\pm}+\frac{1}{\eta}\frac{\partial x_{1,\pm}}{\partial \eta}+\frac{C}{\eta^2}\right] = \frac{1}{\eta} x_{1,\pm}-\frac{C}{\eta}\;.
 \end{equation}
 By the same computation as in equation \eqref{eq:tv-mono-derivative} we have 
 \begin{equation}
 \frac{\partial h_+(C,\eta)}{\partial \eta} = \frac{1}{\eta}\left((x_{1,+}-C)\varphi(x_{1,+})-(x_{1,-}-C)\varphi(x_{1,-})\right)\;.
 \label{eq:etaplusone-derivative}
 \end{equation}
 Note that 
 \begin{align}
 x_{1,+}-C &=\frac{-C(\eta+1)-C(\eta^2-1)+\eta\sqrt{C^2(\eta+1)^2+2(\eta^2-1)\log \eta}}{\eta^2-1} \\
 &= \frac{-C\eta(\eta+1)}{\eta-1}+\frac{\eta\sqrt{C^2(\eta+1)^2+2(\eta^2-1)\log \eta}}{\eta^2-1}>0 \\
 x_{1,-}-C &= \frac{-C\eta(\eta+1)}{\eta-1}-\frac{\eta\sqrt{C^2(\eta+1)^2+2(\eta^2-1)\log \eta}}{\eta^2-1}<0
 \end{align}
 Hence the derivative in \eqref{eq:etaplusone-derivative} is positive. By the same argument, it holds that $h_+(C,\eta) \geq 1-2\Phi(-C) > 1-\rho$. 

 A similar argument can be applied to the case $\mu_1 > \mu_2$ and is omitted. For the equality case, when $\sigma_1 = \sigma_2$ and $C=C_0(\rho)/2$, it is the same as in the proof in \lemmaref{lem:tv-mono}.
\end{proof}

\begin{lemma}
    Suppose $P_{1,2} = N_d(\mu_{1,2},\sigma_{1,2}^2I_d),0<\rho<1$. Let $F_d$ be the CDF of Gamma$(\frac{d}{2},\frac{d}{2})$, and $\eta_0(\rho)$ is the solution of 
    \begin{equation}
    1-\rho = F_d\left(\frac{2\eta^2 \log \eta}{\eta^2-1}\right) - F_d\left(\frac{2 \log \eta}{\eta^2-1}\right)\;.
    \end{equation}
    If $\eta:=\max\{\sigma_1/\sigma_2,\sigma_2/\sigma_1\} > \eta_0(\rho)$,
    then there exists a set $A$ such that $P_1(A)-P_2(A) > 1-\rho$. 
    \label{lem:eta-upperbound}
   \end{lemma}

To prove \lemmaref{lem:eta-upperbound}, we need to invoke a generalization of Reynolds' transportation theorem. The following lemma and required definitions can be found as equation 7.2 in \citet{Differential1973}

\begin{lemma}
 Let $\Omega_t$ be an $\ell-$dimensional time-variant domain of integration in $\mathbb{R}^d$, and can be given by the image of a smooth map $(u,t)\rightarrow x(u,t)$, where $u$ runs over a fixed domain in a $\mathbb{R}^\ell$. Let $\omega$ be an exterior $\ell-$form that can be represented in local coordinates as
 \begin{equation}
 \omega = \sum_H a_H(x,t)dx^H,\quad dx^H = dx^{h_1}\wedge \cdots \wedge dx^{h_\ell}\;.
 \end{equation}
 where $H$ is the multi-index and $1\leq h_1 < h_2 <\cdots < h_\ell \leq n$. Then 
 \begin{equation}
 \frac{d}{dt} \int_{\Omega_t}\omega = \int_{\Omega_t} i_{\bm{v}}(d_x\omega) + \int_{\Omega_t}\dot{\omega} + \int_{\partial \Omega_t}i_{v}\omega
 \label{eq:reynold}
 \end{equation}
 where $v=\partial x/ \partial t$, $i_{v}$ denotes the interior product with $v$, $d_x\omega$ is the exterior derivative of $\omega$ with respect to $x$ and $\dot{\omega} = \sum_H \frac{\partial a_H(x,t)}{\partial t} dx^H $
 \label{lem:reynold}
\end{lemma}

\begin{proof}[Proof of lemma \ref{lem:eta-upperbound}]
 The proof is with the similar but more complicated technique as in the proof of lemma \ref{lem:tv-mono}. WLOG assume that $\sigma_2=\eta\sigma_1 \geq \eta_0(\rho)\sigma_1>\sigma_1$ and 
 \begin{equation}
 \mu_1 = (\mu_{11},\cdots,\mu_{1d}),\quad \mu_2 = (C\sigma_2+\mu_{21},\cdots,\mu_{2d}).
 \end{equation}
 By symmetricity, we can further assume $C\geq 0,\mu_{22}=\mu_{12},\cdots,\mu_{2d}=\mu_{1d}$.

 Let $p_1,p_2$ be the densities of $P_1,P_2$ respectively. Note that $p_1$ is just the density of standard $d-$dimensional Gaussian, which we will denote as $\varphi_d$. Then algebraic computation shows that the set 
 \begin{equation}
 A=\{x\in\mathbb{R}^d: p_1(x) \geq p_2(x)\}=\left\{x\in\mathbb{R}^d: \norm{x-\frac{\eta^2\mu_1 - \mu_2}{\eta^2-1}}\leq \frac{\sigma_2}{\eta^2-1}\sqrt{C^2\eta^2+2d(\eta^2-1)\log\eta}\right\}\;.
 \end{equation}
 Note that $A$ is a ball in $d-$dimension. 
 Write
 \begin{align}
 A_1(C,\eta) &:= \left\{x\in\mathbb{R}^d: \norm{x-(-\frac{C\eta}{\eta^2-1},0,\cdots,0)}\leq \frac{\eta}{\eta^2-1}\sqrt{C^2\eta^2+2d(\eta^2-1)\log\eta}\right\}\;,
 \end{align}
 we can compute the probability directly by
 \begin{align}
 \notag
 P_1(A)
 &= \int_{x:\norm{x-\frac{\eta^2\mu_1 - \mu_2}{\eta^2-1}}\leq \frac{\sigma_2}{\eta^2-1}\sqrt{C^2\eta^2+2d(\eta^2-1)\log\eta}} \frac{1}{(2\pi)^{\frac{d}{2}}\sigma_1^d}\exp\left(-\frac{\norm{x-\mu_1}^2}{2\sigma_1^2}\right) dx \\
 &=\int_{y:\norm{y+\frac{C\eta}{\eta^2-1}}\leq \frac{\eta}{\eta^2-1}\sqrt{C^2\eta^2+2d(\eta^2-1)\log\eta}} \frac{1}{(2\pi)^{\frac{d}{2}}}\exp\left(-\frac{\norm{y}^2}{2}\right) dy \\
 &=\int_{A_1(C,\eta)}\varphi_d(x)dx\;,
 \end{align}
 where $\varphi_d$ is the density of $N_d(0,I_d)$ distribution, and 
 Similarly let $A_2(C,\eta) = \{x\in\mathbb{R}^d:\norm{x-({-C\eta^2}/({\eta^2-1}),0,\cdots,0)} \leq \sqrt{C^2\eta^2+2d(\eta^2-1)\log \eta}/(\eta^2-1) \}$ we have 
 \begin{equation}
 P_2(A) = \int_{A_2(C,\eta)} \varphi_d(x)dx\;. 
 \end{equation}
 Let $h(C,\eta)=P_1(A)-P_2(A)$, as the latter is only determined by $C$ and $\eta$. We now prove that for any fixed $\eta \geq \eta_0(\rho)$, $h(C,\eta)$ is monotonically increasing in $C\geq 0$.
 
 To show this, we consider taking partial derivative of $h(C,\eta)$ with respect to $C$.
 \begin{align}
 \frac{\partial}{\partial C}h(C,\eta) &= \frac{\partial }{\partial C}\int_{A_1(C,\eta)}\varphi_d(x)dx - \frac{\partial }{\partial C} \int_{A_2(C,\eta)}\varphi_d(x)dx\;,
 \end{align}
 We first tackle $\frac{\partial}{\partial C}\int_{A_1(C,\eta)}\varphi_d(x)dx$. $A_1(C,\eta)$ is a hyperball. Let $n$ be the normal vector given in polar coordinates, then one can perform the following change of variable
 \begin{align}
 x:(n,r) &\rightarrow \mathbb{R}^d\\
 x_1 &= o_1+r\cos \theta_1 \\
 x_2 &= r\sin \theta_1 \cos \theta_2 \\
 &\cdots \\
 x_{d-2} &= r\sin \theta_1 \cdots \sin\theta_{d-2} \cos\theta_{d-1} \\
 x_d &= r\sin\theta_1\cdots \sin\theta_{d-2}\sin\theta_{d-1}
 \end{align}
 where $o_1=-C\eta/(\eta^2-1)$ and $0\leq r\leq R_1$ with
 \begin{equation}
 R_1=\frac{\eta}{\eta^2-1}\sqrt{C^2\eta^2+2d(\eta^2-1)\log \eta}\;.
 \end{equation}
 For each $C>0$, there is a small neighborhood of $C\in (C_a,C_b)$ such that $x=(x_1,\cdots,x_d)$ is a continuously differentiable map on $[0,R_1]\times [0,\pi]^{d-2}\times [0,2\pi]\times (C_a,C_b) \rightarrow \mathbb{R}^d$. Here is where lemma \ref{lem:reynold} kicks in. Consider the $d-$form $\omega = \varphi_d(x) dx_1\wedge \cdots dx_d$. Then since $\omega$ is a $d-$form on $\mathbb{R}^d$, its exterior derivative w.r.t. $x\in\mathbb{R}^d$ is zero. And since $\varphi_d$ is independent of $C$, $\dot{\omega}=0$. 

 Therefore equation \eqref{eq:reynold} is simplified to 
 \begin{equation}
 \frac{\partial}{\partial C}\int_{A_1(C,\eta)}\omega = \int_{\partial A_1(C,\eta)} i_{v_1}(\omega).
 \end{equation}
 where $v_1=\partial x/\partial C$.

 We shall now identify $i_{v_1}(\omega)$. As $\partial A_1(C,\eta)$ is a hypersphere with dimension $d-1$, it is orientable. Further, consider the following bases of tangent space of $\partial A_1(C,\eta)$
 \begin{equation}
 e_i = \frac{\partial}{\partial \theta_i},\quad 1\leq i \leq d-1
 \end{equation}
 and outward pointing unit normal $n=\partial/\partial r$. Then for any vector field $t=\sum_{1\leq i\leq d-1} t_ie_i+t_d n$, by the definition of interior product and the fact that $\omega$ is an alternating $d-$linear tensor, it holds that 
 \begin{align}
 i_{t}(\omega)(e_1,\cdots,e_{d-1})
 =&\omega(t,e_1,\cdots,e_{d-1})\\
 =&\sum_{i=1}^{d-1}t_i\omega(e_i,e_1,\cdots,e_{d-1})+ t_d\omega(n,e_1,\cdots,e_{d-1})\\
 =& \varphi_d(x)t_d (dx_1 \wedge \cdots dx_d)(n,e_1,\cdots,e_{d-1})\\
 =& \varphi_d(x)\langle t,n\rangle d\Sigma(n)(e_1,\cdots,e_{d-1})
 \end{align}
 where we denote $d\Sigma(n)=i_{n}(dx_1\wedge \cdots \wedge dx_n)$ to be the area element. In other words, $d\Sigma(n)$ is the differential form such that 
 \begin{align}
 d\Sigma(n)(e_1,\cdots,e_{d-1})&= (dx_1\wedge \cdots\wedge dx_d)(n,e_1,\cdots,e_{d-1}) \\
 &=R_1^{d-1}\sin^{d-2}\theta_1\cdots\sin\theta_{d-2}(dr\wedge d\theta_1 \wedge\cdots \wedge d \theta_{d-1})(n,e_1,\cdots,e_{d-1})\\
 &=R_1^{d-1}\sin^{d-2}\theta_1\cdots\sin\theta_{d-2}\;,
 \end{align}
 which shows that $d\Sigma(n)=R_1^{d-1}\sin^{d-2}\theta_1\cdots\sin\theta_{d-2}d\theta_1\wedge \cdots\wedge d \theta_{d-1}$.
 Then since $i_{v_1(\omega)}=\varphi_d(x)\langle v_1,n\rangle d\Sigma(n)$,
 \begin{align}
 \frac{\partial}{\partial C}\int_{A_1(C,\eta)}\omega &
 =\int_{\mathbb{S}^{d-1}}R_1^{d-1}\varphi_d({x}({n},R_1)) \langle {v}_1, {n}\rangle d\Sigma({n})\;. 
 \end{align}

 Similarly $A_2(C,\eta)$ is given by 
 \begin{align}
y:(n,r) &\rightarrow \mathbb{R}^d\\
 y_1 &= o_2+r\cos \theta_1 \\
 y_2 &= r\sin \theta_1 \cos \theta_2 \\
 &\cdots \\
 y_{d-1} &= r\sin \theta_1 \cdots \sin\theta_{d-2} \cos\theta_{d-1} \\
 y_d &= r\sin\theta_1\cdots \sin\theta_{d-2}\sin\theta_{d-1}
 \end{align}
 with $o_2 = -C\eta^2/(\eta^2-1)$ and $0\leq r \leq R_2$, where 
 \begin{equation}
 R_2 = \frac{1}{\eta^2-1}\sqrt{C^2\eta^2+2d(\eta^2-1)\log\eta}\;.
 \end{equation}
 By the same argument, we have 
 \begin{equation}
 \frac{\partial}{\partial C}\int_{A_2(C,\eta)}\omega = \int_{\partial A_2(C,\eta)} \varphi_d({y}({n},R_2))\langle {v}_2,{n}\rangle d\Sigma({n})=\int_{\mathbb{S}^{d-1}}R_2^{d-1}\varphi_d({y}({n},R_2))\langle {v}_2,{n}\rangle d\Sigma({n})\;.
 \end{equation}
 where $v_2 = \partial {y}/\partial C$.

 We observe two identites. First, for the same $$n=(\cos\theta_1,\sin\theta_1\cos\theta_2,\cdots,\sin\theta_1\cdots\sin\theta_{d-2}\cos\theta_{d-1},\sin\theta_1\cdots\sin\theta_{d-1})$$ it holds that
 \begin{align}
 \notag
 &\frac{\varphi_d({x}({n},R_1))}{\varphi_d({y}({n},R_2))}\\
 =& \frac{\exp\left(-\frac{1}{2}\norm{{x}({n},R_1)}^2\right)}{\exp\left(-\frac{1}{2}\norm{{y}({n},R_2)}^2\right)} \\
 =& \frac{\exp\left(-\frac{1}{2}\left[o_1^2+R_1^2+2o_1R_1\cos \theta_1\right]\right)}{\exp\left(-\frac{1}{2}\left[o_2^2+R_2^2+2o_2R_2\cos \theta_1\right]\right)} \\
 =& \frac{\exp\left(-\frac{1}{2}\left[\frac{C^2\eta^2}{(\eta^2-1)^2}+\frac{\eta^2}{(\eta^2-1)^2}(C^2\eta^2+2d(\eta^2-1)\log\eta)-2\frac{C\eta^2}{(\eta^2-1)^2}\sqrt{C^2\eta^2+2d(\eta^2-1)\log\eta}\cos\theta_1 \right]\right)}{\exp\left(-\frac{1}{2}\left[\frac{C^2\eta^4}{(\eta^2-1)^2}+\frac{1}{(\eta^2-1)^2}(C^2\eta^2+2d(\eta^2-1)\log\eta)-2\frac{C\eta^2}{(\eta^2-1)^2}\sqrt{C^2\eta^2+2d(\eta^2-1)\log\eta}\cos\theta_1\right]\right)}\\
 =&\exp\left(-\frac{1}{2}2d\log \eta\right)\\
 =&\frac{1}{\eta^d}
 \end{align}
 Second, for same $\bm{n}$, it holds that $\bm{x}(\bm{n},R_1)-\eta\bm{y}(\bm{n},R_2)=(C\eta,0,\cdots,0)$. This shows $\bm{v}_1-\eta\bm{v}_2=(\eta,0,\cdots,0)$
 
 Therefore 
 \begin{align}
 &\frac{\partial h(C,\eta)}{\partial C} \\
 =& \int_{\mathbb{S}^{d-1}}\left(R_1^{d-1}\varphi_d(\bm{x}(\bm{n},R_1))\langle \bm{v}_1,\bm{n}\rangle - R_2^{d-1}\varphi_d(\bm{y}(\bm{n},R_1))\langle \bm{v}_2,\bm{n}\rangle \right)d\Sigma(\bm{n}) \\
 =&\int_{\mathbb{S}^{d-1}} \left(R_1^{d-1}\varphi_d(\bm{x}(\bm{n},R_1))\langle \bm{v}_1,\bm{n}\rangle - \left(\frac{R_1}{\eta}\right)^{d-1}\eta^d\varphi_d(\bm{x}(\bm{n},R_1))\langle \bm{v}_2,\bm{n}\rangle \right)d\Sigma(\bm{n}) \\
 =&\int_{\mathbb{S}^{d-1}}R_1^{d-1}\varphi_d(\bm{x}(\bm{n},R_1))\langle \bm{v}_1-\eta \bm{v}_2,\bm{n}\rangle d\Sigma(\bm{n}) \\
 =&\int_{\mathbb{S}^{d-1}}R_1^{d-1}\eta\cos \theta_1 \varphi_d(\bm{x}(\bm{n},R_1))d\Sigma(\bm{n})\\
 =&R_1^{d-1}\eta \int_0^\pi \cdots \int_0^\pi \int_0^{2\pi} \cos \theta_1 \frac{1}{(2\pi)^{\frac{d}{2}}}\exp\left(-\frac{1}{2}[o_1^2+R_1^2+2o_1R_1\cos\theta_1]\right) \sin^{d-2}\theta_1 \cdots \sin\theta_{d-2}d\theta_1d\theta_2\cdots d\theta_{d-1}\\
 =&\Xi \int_0^\pi \cos \theta_1 \exp(-o_1R_1\cos\theta_1)\sin^{d-2}\theta_1 d\theta_1\;,
 \end{align}
 where $\Xi$ is a positive constant. Finally, 
 \begin{align}
 &\quad\int_0^\pi \cos \theta_1 \exp(-o_1R_1\cos\theta_1)\sin^{d-2}\theta_1 d\theta_1 \\&= \int_0^{\frac{\pi}{2}} \cos \theta_1 \exp(-o_1R_1\cos\theta_1)\sin^{d-2}\theta_1 d\theta_1 + \int_{\frac{\pi}{2}}^{\pi} \cos \theta_1 \exp(-o_1R_1\cos\theta_1)\sin^{d-2}\theta_1 d\theta_1 \\ 
 &=\int_0^{\frac{\pi}{2}} \cos \theta_1 \exp(-o_1R_1\cos\theta_1)\sin^{d-2}\theta_1 d\theta_1 - \int_0^{\frac{\pi}{2}}\cos \theta_1 \exp(o_1R_1\cos\theta_1)\sin^{d-2}\theta_1 d\theta_1 \\
 &=\int_0^{\frac{\pi}{2}} \cos\theta_1 \sin^{d-2}\theta_1 \left[ \exp(-o_1R_1\cos\theta_1) - \exp(o_1R_1\cos\theta_1) \right] d\theta_1 > 0
 \end{align}
 where the last inequality is due to the fact that $o_1<0$ and for any $x>0,e^x > 1 > e^{-x}$. This shows that when $C > 0$, $h(C,\eta)$ is increasing in $C$. Therefore as long as $h(0,\eta) > 1-\rho$, $P_1(A)-P_2(A)>1-\rho$. It remains to show that $h(0,\eta)\geq 1-\rho$ as long as $eta\geq \eta_0$. This is direct, when $C = 0$, 
 \begin{align}
 A_1 &=\left\{x\in\mathbb{R}^d: \norm{x}\leq \sqrt{\frac{2d\eta^2\log\eta}{\eta^2-1}}\right\} \\ 
 A_2 &=\left\{x\in\mathbb{R}^d: \norm{x}\leq \sqrt{\frac{2d\log\eta}{\eta^2-1}}\right\}
 \end{align}
 Let $X\sim N_d(0,I_d)$, then $\norm{X}^2/d\sim Gamma(\frac{d}{2},\frac{d}{2})$. Then
 \begin{equation}
 h(0,\eta)=Pr\left(\norm{X}^2 \leq \frac{2d\eta^2\log\eta}{\eta^2-1} \right)- Pr\left(\norm{X}^2 \leq \frac{2d\log\eta}{\eta^2-1} \right) = F_d(\frac{2\eta^2\log\eta}{\eta^2-1}) - F_d(\frac{2\log\eta}{\eta^2-1})\;.
 \end{equation}
 By \lemmaref{lem:mono-eta-function}, $h(0,\eta)$ is increasing in $\eta$. Therefore $h(0,\eta)\geq h(0,\eta_0(\rho))=1-\rho$,
 and this completes the proof.
\end{proof}

\subsubsection{Component-wise comparison theorem}
The second key ingredient for proving our result is a group of
component-wise comparison lemmas on two mixture distributions with
small total variation distance. In the following lemmas, consider two
mixture distributions $P=\sum_{i=1}^K \pi_i P_i$ and
$P'=\sum_{j=1}^{K'} \pi_j' P_j$ and denote $\pi_{\min} =
\min_{i,j}\{\pi_i,\pi_j'\}$ and
$\pi_{\max}=\max_{i,j}\{\pi_i,\pi_j'\}$.  The proofs can be found in
\sectionref{sec:proof-main}; these lemmas essentially show that (1)
the weighted sum of component-wise total variation difference is upper
bounded if the two mixtures are close and (2) when two well-separated
mixtures have same number of components, under particular
correspondence conditions, their mixture proportions are close.

\begin{lemma}
  \label{lem:component-wise-tv-P}
  Suppose $TV(P,P')\leq 2\epsilon$. For any $j\in[K']$, if there exists a collection of sets $\{A_{ij}\}_{i\in[K]}$ such that $\rho_{ij}=1-P_i(A_{ij})+P_j'(A_{ij})\in [0,1]$ for all $i\in[K]$, then
   \begin{equation}
   \sum_{i=1}^K\max\{\pi_i,\pi_j'\}\rho_{ij}\geq \pi_j'-2\epsilon\;,\quad\sum_{i=1}^K \pi_i\rho_{ij}\geq \pi_{\min}-2\epsilon\;.
   \label{eq:component-wise-tv-P}
   \end{equation}
   Similarly, if there exists a collection of sets $\{B_{ij}\}_{i\in[K]}$ such that $\rho_{ij}=1-P_j'(B_{ij})+P_i(B_{ij})\in [0,1]$ for all $i\in[K]$, equation \label{eq:component-wise-tv-P2} also holds.
\end{lemma}
This lemma stems from the fact that there cannot be a set that has a
large probability mass of a component in $P_i$ of $P$ but has small
probability mass from every component $P'_j$ of $P'$. $A_{ij}$ is the set that has a large probability mass under $P_i$ but small under $P_j'$, while $B_{ij}$ vice versa. A symmetric result can be obtained by switching the role of $P$ and $P'$.


\lemmaref{lem:component-wise-tv-P} is particularly useful when,
for some component $P_j'$, all components $P_i,\,i\neq j$ are far away from $P_j'$ in total variation distance; then by \lemmaref{lem:component-wise-tv-P} one can
upper bound the total variation distance between $P_j'$ and the
remaining component $P_j$ of $P$.
\begin{lemma}
  Suppose $TV(P,P')\leq 2\epsilon$ and denote $\pi_{\min} = \min_{i,j}\{\pi_i,\pi_j'\}$.\comment{ Further suppose $\max\{K,K'\}\leq 1/\pi_{\min}$ and $\max\{K,K'\}\leq 1/\pi_{\min}$.} Consider two components $P_{j_0},P_j'$ with $\pi_{j_0}\leq \pi_{j}'$. If for all $i\in[K],i\neq j_0$, there is a set $A_{ij}$ such that $P_i(A_{ij})-P_j'(A_{ij}) \geq 1-\rho$ with some constant $\rho$ satisfying
 $0\leq \rho \leq (\pi_{\min}-2\epsilon)/(1-\pi_{\min})$, it holds that
 \begin{equation}
 TV(P_{j_0},P_j') \leq \frac{2\epsilon}{\pi_{\min}}+\frac{1-\pi_{\min}}{\pi_{\min}}\rho\;.
 \label{eq:tv-bound-unknown-pi}
 \end{equation}
 \label{lem:rho1-rho2}
\end{lemma}
Now, we obtain a component-wise comparison result for the differences in mixture proportions between corresponding components. 

\begin{lemma}
  \label{lem:diff-proportion}
  Suppose $TV(P,P')\leq 2\epsilon$. Denote $\pi_{\min} = \min_{i,j}\{\pi_i,\pi_j'\}$ and $\pi_{\max} = \max_{i,j}\{\pi_i,\pi_j'\}$  If there is a pair of components $P_i,P_i'$ and a set $A$ such that for some $\rho\in (0,1)$, $\min\{P_i(A),P_i'(A)\}\geq 1-\rho$ and $\max_{j\neq i}\{P_j(A),P_j'(A)\} \leq \rho$, then
  \begin{equation}
    |\pi_i-\pi_i'|\leq 2\epsilon + (1+\pi_{\max}-\pi_{\min})\rho
  \end{equation}
\end{lemma}
\begin{proof}[Proof of \lemmaref{lem:component-wise-tv-P}]
Let $Q=(P+P')/2$. Then $TV(P,Q)\leq \epsilon$, $TV(P',Q)\leq \epsilon$.
 Let $P_j'$ be an arbitrary component of $P'$. We will first show by contradiction the first half of \eqref{eq:component-wise-tv-P} holds. Suppose that $\sum_{i=1}^K\max\{\pi_i,\pi_j'\}\rho_{ij} < \pi_j'-2\epsilon$. Now consider the set $A=\cap_{i=1}^K A_{ij}^\complement$. On one hand,
 \begin{align}
 Q(A) \geq \pi_j'P_j'((\cup_{i=1}^K A_{ij})^\complement) - \epsilon\geq \pi_j'(1-\sum_{i=1}^K P_j'(A_{ij}))-\epsilon = \pi_j'-\pi_j'\sum_{i=1}^K P_j'(A_{ij}) - \epsilon:=LB
 \end{align}
 Meanwhile,
 \begin{align}
 Q(A) \leq \sum_{i=1}^K \pi_i P_i(A_i^\complement) + \epsilon=1-\sum_{i=1}^K\pi_i P_i(A_i)+\epsilon:=UB
 \end{align}
 Since for every $i\in[K]$ it holds that $P_i(A_{ij}) - P_j'(A_{ij}) = 1-\rho_{ij}$, then by \lemmaref{lem:weightedTotalVariation}, $\pi_i P_i(A_{ij}) - \pi_j' P_j'(A_{ij}) \geq \pi_i - \max\{\pi_i,\pi_j'\}\rho_{ij}$. Therefore, applying \lemmaref{lem:weightedTotalVariation} to all components $i\in[K]$, it holds that
 \begin{align}
 LB - UB 
 &= \pi_{j}'-1-2\epsilon+\sum_{i=1}^K \left( \pi_i P_i(A_{ij})-\pi_j'P_j'(A_{ij})\right) \\
 &\geq \pi_{j}'-1 -2\epsilon +\sum_{i=1}^K 
 \left(\pi_i-\max\{\pi_i,\pi_j'\}\rho_{ij}\right)\\
 &= \pi_{j}'-1 -2\epsilon +1- \sum_{i=1}^K \max\{\pi_i,\pi_j'\} \rho_{ij} > 0 \;.
 \end{align}
 This is a contradiction since $Q(A) \geq LB > UB \geq Q(A)$. So the first half of \eqref{eq:component-wise-tv-P} holds for $P_j'$.\\

 For the second half of \eqref{eq:component-wise-tv-P}, if $\sum_{i=1}^K \pi_i\rho_{ij} < \pi_{\min}-2\epsilon$, then consider a different lower bound of $Q(A)$
\begin{align}
    Q(A)\geq \pi_j'P_j'((\cup_{i=1}^K A_{ij})^\complement) - \epsilon \geq \pi_{\min} P_j'((\cup_{i=1}^K A_{ij})^\complement) - \epsilon \geq \pi_{\min} -\pi_{\min}\sum_{i=1}^K P_j'(A_{ij})-\epsilon := LB';,
\end{align}
then notice that
 \begin{align}
 LB'-UB &\geq \pi_{\min}-1-2\epsilon+\sum_{i=1}^K\pi_i P_i(A_{ij}) - \pi_{\min}P_j'(A_{ij})\\
 &\geq \pi_{\min}-1-2\epsilon+\sum_{i=1}^K\pi_i(1-\rho_{ij}) \\
 &=\pi_{\min}-1-2\epsilon+1-\sum_{i=1}^K\pi_i\rho_{ij}>0\;,
 \end{align}
 which is a contradiction and completes the proof of the second half of \eqref{eq:component-wise-tv-P}.
\end{proof}

\begin{proof}[Proof of \lemmaref{lem:rho1-rho2}]
 Define 
 \begin{equation}
 \rho_1 = \frac{\pi_{\min}-2\epsilon}{\pi_{\min}}-\frac{1-\pi_{\min}}{\pi_{\min}}\rho
 \end{equation}
 Note that $0\leq \rho\leq (\pi_{\min}-2\epsilon)/(1-\pi_{\min})$ implies that $0\leq \rho_1 \leq 1$. We will show by contradition that the total variation distance $TV(P_{j_0},P_j')\leq 1-\rho_1$. If otherwise $TV(P_{j_0},P_j') > 1-\rho_1$, one can select a set $A_{j_0j}$ such that $P_{j_0}(A_{j_0j}) - P_j'(A_{j_0j}) > 1-\rho_1$. Define function
 \begin{equation}
 f(x)=x(\rho_1 - 1) + \sum_{\substack{i\in[K] \\ i\neq j_0}} \max\{\pi_i,x\}\rho\;,
 \end{equation}
 then $f(x)$ is a piecewise linear function, with its slope upper bounded by
 \begin{equation}
 \rho_1 - 1 + \sum_{i:i\neq j} \rho =\rho_1 + (K-1)\rho -1 \leq \rho_1 + \left(\frac{1-\pi_{\min}}{\pi_{\min}}\right)\rho - 1 =\frac{\pi_{\min}-2\epsilon}{\pi_{\min}}-1 < 0\;.
 \end{equation}
 Hence $f(x)$ is decreasing in $x$. Since $\pi_{j_0}\leq \pi_j'$, according to the remark after \lemmaref{lem:component-wise-tv-P}, we have 
 \begin{align}
 f(\pi_j')&=\pi_j'\rho_1 + \sum_{\substack{i\in[K] \\ i\neq j_0}} \max\{\pi_i,\pi_j'\}\rho - \pi_j' \\
 & >\sum_{i=1}^K \max\{\pi_i,\pi_j'\} (1-P_i(A_{ij})+P_j'(A_{ij})) -\pi_j' \\
 &\geq -2\epsilon\;. 
 \end{align}
 Then we conclude that $f(\pi_{\min}) \geq f(\pi_j')> -2\epsilon$. But we can explicitly compute that 
 \begin{align}
 f(\pi_{\min}) & = \pi_{\min} \left(\frac{\pi_{\min}-2\epsilon}{\pi_{\min}}-\frac{1-\pi_{\min}}{\pi_{\min}}\rho-1\right)+\sum_{\substack{i\in[K] \\ i\neq j_0}} \max\{\pi_i,\pi_{\min}\}\rho \\
 &= -2\epsilon-(1-\pi_{\min})\rho + \rho\sum_{\substack{i\in[K] \\ i\neq j_0}} \pi_i \\
 &\leq -2\epsilon -(1-\pi_{\min})\rho + (1-\pi_{\min})\rho \\
 & = -2\epsilon\;.
 \end{align}
 This contradiction completes the proof of \eqref{eq:tv-bound-unknown-pi}. 
\end{proof}

\begin{proof}Proof of \lemmaref{lem:diff-proportion}]
    Note that $\pi_i(A^\complement)\leq \rho$ and $\pi_i'(A^\complement)\leq \rho$
    \begin{align}
        |\pi_i-\pi_i'| &= |\pi_i P_i(A) - \pi_i' P_i'(A) + \pi_i P_i(A^\complement) - \pi_i' P_i'(A^\complement)| \\
        &\leq |P(A)- \sum_{j\neq i }\pi_j P_j(A) - P'(A) + \sum_{j\neq i }\pi_j' P_j'(A) | + |\pi_i P_i(A^\complement) - \pi_i' P_i'(A^\complement)| \\
        &\leq 2\epsilon + \max\{\sum_{j\neq i }\pi_j P_j(A),\sum_{j\neq i }\pi_j' P_j'(A) \}+\max\{\pi_i P_i(A^\complement),\pi_i' P_i'(A^\complement)\}\\
        &\leq 2\epsilon + (1-\pi_{\min}+\pi_{\max})\rho\;.
    \end{align}
\end{proof}

\subsection{Initialization}
\label{sec:initial}
Let $c_0$ be defined as in \eqref{eq:c0}, then the following theorem guarantees that for each component of $P'$, there exists a component of $P$ that is close to it in total variation distance. Starting from here and in the remaining of \sectionref{sec:proof-main}, we will assume the following:
\begin{itemize}
    \item[B1] $P\in\mathcal{M}(K,\pi_{\min},\pi_{\max},c)$ and $P'\in\mathcal{M}(K',\pi_{\min},\pi_{\max},c)$.
    \item[B2] There exists a distribution $Q$ such that $TV(P,Q\leq \epsilon$ and $TV(P',Q)\leq \epsilon$.
    \item[B3] $\max\{K,K'\}\leq 1/\pi_{\min}$, $\pi_{\max}\leq 1-(\min\{K,K'\}-1)\pi_{\min}$.
\end{itemize}


\begin{theorem}
  Suppose B1-B3 hold.
  For any component $P_j'$, there must be a component $P_i$ such that $TV(P_i,P_j') \leq 1-\pi_{\min}+2\epsilon$
 \label{thm:step1-max-i-j}
\end{theorem}

\begin{proof}[Proof of \theoremref{thm:step1-max-i-j}]
 WLOG we prove the statement for $P_1'$. Assume that the conclusion of the theorem does not hold, that for any component $P_i$, it holds that $TV(P_i,P_1') > 1-\pi_{\min}+2\epsilon$, implying the existence of a set $A_i$ such that $P_i(A_i)-P_1'(A_i) > 1-\pi_{\min}+2\epsilon$, or $\pi_{\min}-2\epsilon > 1-P_i(A_i)+P_1'(A_i)$. Then applying \lemmaref{lem:component-wise-tv-P} to the collection of sets $\{A_i\}_{i\in[K]}$, we have
 \begin{equation}
 \pi_{\min}-2\epsilon > \sum_{i=1}^K \pi_i(1-P_i(A_i)+P_1(A_i'))\geq \pi_{\min}-2\epsilon\;.
 \end{equation} 
 Therefore, the desired result holds.
\end{proof}

The second theorem states that, if both $P$ and $P'$ are well-separated and have nearly balanced components, then the matching established in theorem \ref{thm:step1-max-i-j} is one-to-one. Specifically, for any component in $P'$, there is one and only one component in $P$ such that their centers are close.

\begin{theorem}
Suppose B1-B3 hold,
If $c>\eta_0c_0$, where $c_0,\eta_0$ are defined in \eqref{eq:c0} and \eqref{eq:eta0-diff} respectively, then $K=K'$ and there is a permutation $\theta:[K]\rightarrow [K]$ such that $\norm{\mu_i-\mu_{\theta(i)}'}\leq c_0\max\{\sigma_i,\sigma_{\theta(i)}'\}$ and $\max\{\sigma_i/\sigma_{\theta(i)}',\sigma_{\theta(i)}'/\sigma_i\}\leq \eta_0$.
 \label{thm:unique-matching}
\end{theorem}

\begin{proof}[Proof of \theoremref{thm:unique-matching}]
 According to \theoremref{thm:step1-max-i-j} and \lemmaref{lem:tv-mono} each $P_j'$ must be matched to at least one component $P_i$ in the sense that their center distance $\norm{\mu_i-\mu_j'}\leq c_0\max\{\sigma_i,\sigma_j'\}$. Note that \theoremref{thm:step1-max-i-j} is symmetric on both $P$ and $P'$.  For each $P_i$, there must be a $P_j'$ such that their centers are close. 
 
 We now considering the following removing procedure. For a component $P_j$, if there is a unique $P_j'$ that is matched to $P_j$, and $P_j$ is the only component in $P$ that is matched to $P_j'$, then remove both $P_j$ and $P_j'$. For a removed pair $P_j,P_j'$, one can upper bound the ratio between their standard deviations by $\max\{\sigma_j,\sigma_j'\}\leq \eta_0$. To see this WLOG we assume that $\sigma_j \leq \sigma_j'$. Then for any $i\neq j$, let $P_i'$ be a component matched with $P_i$. The distances between centers are
 \begin{equation}
 \norm{\mu_i'-\mu_j} \geq \norm{\mu_i'-\mu_j'} - \norm{\mu_j'-\mu_j} \geq c(\sigma_i'+\sigma_j')-c_0\sigma_j' =c\sigma_i' + (c-c_0)\sigma_j' \;.
 \label{eq:dist-i-j-first}
 \end{equation}
 \begin{itemize}
 \item If $\sigma_i'\geq \sigma_j$, it holds that $\norm{\mu_i'-\mu_j}\geq c\max\{\sigma_i',\sigma_j\}$. By \lemmaref{lem:tv-mono}, one can select a set $A_{ij}$ such that $P_i'(A_{ij})-P_j(A_{ij}) > 1-2\Phi(-c/2)\geq 1-2\Phi(-c_0\eta_0/2)$
 \item If $\sigma_i' < \sigma_j$, then since $P_i'$ is not matched with $P_j$ 
 \begin{equation}
 \norm{\mu_i'-\mu_j}\geq c_0\max\{\sigma_i',\sigma_j\}=c_0\sigma_j\;.
 \label{eq:dist-i-j-second}
 \end{equation}
 Adding \eqref{eq:dist-i-j-first} and \eqref{eq:dist-i-j-second}, we have 
 \begin{align}
 \norm{\mu_i'-\mu_j} &\geq \frac{1}{2}\left(c\sigma_i' + (c-c_0)\sigma_j' + c_0\sigma_j\right)\\
 &\geq \frac{1}{2}c(\sigma_i'+\sigma_j)\;.
 \end{align}
 By \lemmaref{lem:etaplusone}, one can select a set $A_{ij}$ such that $P_i'(A_{ij})-P_j(A_{ij})\geq 1-2\Phi(-c_0\eta_0/2)$.
 \end{itemize}
 
 Hence in both cases, for any $i\neq j$ there $P_i'(A_{ij})-P_j(A_{ij})\geq 1-\rho_2$ with $\rho_2 = 2\Phi(-c_0\eta_0/2)\in(0,1)$. 
 
 Now suppose $\max\{\sigma_j,\sigma_j'\}/\min\{\sigma_j,\sigma_j'\}>\eta_0$, by \lemmaref{lem:eta-upperbound}, one can select a set $A_{jj}$ such that $P_j'(A_{ij})-P_j(A_{ij}) > 1-\rho_1$, where $\rho_1$ is defined as 
 \begin{equation}
 \rho_1:=\frac{\pi_{\min}-2\epsilon}{\pi_{\max}}-\frac{2(1-\pi_{\max})}{\pi_{\max}}\Phi(-\frac{1}{2}\eta_0 c_0)\;.
 \label{eq:rho1-matched}
 \end{equation}
 Note that $1-\rho_1$ is the left hand side of equation \eqref{eq:eta0-diff}. Further, $\rho_1 < \pi_{\min}/\pi_{\max}<1$ and 
 \begin{equation}
     \rho_1 > (\pi_{\min}-2\epsilon)\left(\frac{1}{\pi_{\max}}-\frac{1-\pi_{\max}}{\pi_{\max}}\right) = \pi_{\min}-2\epsilon > \rho_2
 \end{equation}

 According to \lemmaref{lem:component-wise-tv-P},
 \begin{equation}
 \pi_j'\rho_1 + \sum_{i:i\neq j}\pi_i'\rho_2 > \sum_{i=1}^K \pi_i'(1-P_i'(A_{ij})+P_j(A_{ij}))\geq \pi_{\min}-2\epsilon;.
 \label{eq:removed-eta-bound1}
 \end{equation}
 However the L.H.S of proceeding \eqref{eq:removed-eta-bound1} satisfies that  
 \begin{equation}
 \pi_j'\rho_1 + \sum_{i:i\neq j}\pi_i'\rho_2 = \pi_j'(\rho_1-\rho_2) + \rho_2 \leq \pi_{\max}\rho_1 + (1-\pi_{\max})\rho_2=\pi_{\min}-2\epsilon
 \label{eq:removed-eta-bound2}
 \end{equation}
 The contradiction shows that $\max\{\sigma_j/\sigma_j',\sigma_j'/\sigma_j\}\leq \eta_0$.\\
 
 Suppose now we complete this removing procedure and there are still remaining components, without loss of generosity we can assume that $P_1$ is the component that has smallest standard deviation among all remaining components in both $P$ and $P'$. Since it is not removed, there is a $P_1'$ and a $P_2$ such that both $P_1$ and $P_2$ are matched to $P_1'$. Since we assumed that $\sigma_1\leq \sigma_2$, by triangle inequality
 \begin{equation}
c(\sigma_1+\sigma_2) \leq \norm{\mu_1-\mu_2} \leq \norm{\mu_1'-\mu_1} + \norm{\mu_1'-\mu_2} \leq c_0\max\{\sigma_1,\sigma_1'\}+c_0\max\{\sigma_2,\sigma_1'\}.
 \end{equation}
 Note that $c > c_0\eta_0 > c_0$ by the assumption of the theorem, we consider three cases:
 \begin{itemize}
 \item $\sigma_1'\leq \sigma_1 \leq \sigma_2$. Then $c_0(\sigma_1+\sigma_2)<c(\sigma_1+\sigma_2)\leq c_0(\sigma_1+\sigma_2)$ and this is impossible.
 \item $\sigma_1 < \sigma_1'\leq \sigma_2$. Then $c\sigma_1+c_0\sigma_2 < c(\sigma_1+\sigma_2) \leq c_0\sigma_1'+c_0\sigma_2$, implying that $\sigma_1'/\sigma_1 \geq c/c_0 > \eta_0$.
 \item $\sigma_1\leq \sigma_2 < \sigma_1'$. Then $2c\sigma_1\leq c(\sigma_1+\sigma_2)\leq 2c_0\sigma_1'$, also implying that $\sigma_1'/\sigma_1 \geq c/c_0 > \eta_0$.
 \end{itemize}
 We conclude that $\sigma_1'/\sigma_1 \geq c/c_0$. By \lemmaref{lem:eta-upperbound} and definition of $\eta_0$ in \eqref{eq:eta0-diff}, the total variation distance of $P_1,P_1'$ can be lower bounded by
 \begin{equation}
 TV(P_1,P_1') > F_d(\frac{2\eta_0^2\log\eta_0}{\eta_0^2-1}) - F_d(\frac{2\log\eta_0}{\eta_0^2-1}) = 1-\rho_1\;,
 \end{equation}
 where $\rho_1$ is the same as in \eqref{eq:rho1-matched}.

 On the other hand, notice that $P_1$ cannot be matched with any components in $P'$ other than $P_1'$. This is because if there is a $P_2'$ that is not removed and is matched to $P_1$, by the same argument above one can show that $\sigma_2' < \sigma_1 \leq \sigma_1'$. This contradicts the minimal variance selection of $\sigma_1$. 
 
 Therefore, for any $P_j'$ unremoved with $j\geq 2$, denote $P_j$ be a different component (not $P_1$) in $P$ that matched to $P_j'$, then 
 \begin{align}
 \norm{\mu_j'-\mu_1} &\geq \norm{\mu_j'-\mu_1'}-c_0\max\{\sigma_1,\sigma_1'\} \\
 &\geq c(\sigma_j'+\sigma_1')-c_0\sigma_1' \\
 &\geq c \sigma_j' + (c-c_0)\sigma_1' \\
 &> c\sigma_j'
 \label{eq:lowerbound-unremoved-dist}
 \end{align}
 Also again by selection of $\sigma_1$ $\norm{\mu_j'-\mu_1}>c\sigma_j'\geq c\sigma_1$. Hence $\norm{\mu_j'-\mu_1}>c\max\{\sigma_j',\sigma_1\}$.\\
 
 For any $P_j'$ that has been removed, still denote $P_j$ to be the unique component matched to $P_j'$, \eqref{eq:lowerbound-unremoved-dist} still holds. We now show that $\norm{\mu_j'-\mu_1} > c\max\{\sigma_1,\sigma_j'\}$. If $\sigma_1\leq \sigma_j'$, then it trivially holds. When $\sigma_j' < \sigma_1$, if otherwise $\norm{\mu_j'-\mu_1}\leq c\sigma_1$, then 
 \begin{equation}
 c(\sigma_j+\sigma_1)\leq \norm{\mu_j-\mu_1}\leq \norm{\mu_j-\mu_j'}+\norm{\mu_j'-\mu_1} \leq c_0\max\{\sigma_j,\sigma_j'\}+c\sigma_1\;,
 \end{equation}
 implying that $\max\{\sigma_j,\sigma_j'\}/\sigma_j \geq c/c_0> \eta_0$. This is impossible, therefore for any $P_j'$ that has been removed, $\norm{\mu_1-\mu_j'}\geq c\max\{\sigma_1,\sigma_j'\}$.\\

 Combining the two cases together, again we can find sets $A_{i}$ such that 
 \begin{itemize}
 \item For $i\neq 1$, whether $P_i'$ is removed or not, $P_i'(A_{i})-P_1(A_{i})> 1-\rho_2$.
 \item For $i=1$, $P_1'(A_1)-P_1(A_1) > 1-\rho_1$.
 \end{itemize}
The same arguments in \eqref{eq:removed-eta-bound1} and \eqref{eq:removed-eta-bound2} lead to a contradiction, further showing that all components should have been removed, i.e. when $c>c_0\eta_0$ the match is one-to-one and hence $K=K'$. The upper bounds on center distances and standard deviation ratios have been established earlier.
\end{proof}

\subsection{Iterative Refinements on Mean and Standard Deviations}
\label{sec:refined}
\theoremref{thm:unique-matching} provides conditions on separation, minimal proportion, and maximal proportion such that one can establish a one-to-one correspondence between components of $P$ and $P'$. In this section, we are going to show that once this is done $K=K'$, we can further iteratively improve the bounds. Hence we can complete the proof of \theoremref{thm:main1}. We assume all assumptions in \theoremref{thm:unique-matching} hold.\\

\begin{proof}[Proof of \theoremref{thm:main1}: Upper bounds on mean and standard deviations]
 By $TV(P,P')\leq 2\epsilon$, take $Q=(P+P')/2$, we have $TV(P,Q)\leq \epsilon,TV(P',Q)\leq \epsilon$, therefore we can apply \theoremref{thm:unique-matching}, which confirms a one-to-one correspondence in the sense that centers are close. Out of simplicity in notation, we re-order components of $\mu_j'$ so that $P_i,P_i'$ are correspondence. Specifically, for any pair $P_i,P_i'$, according to the assumption $\norm{\mu_i-\mu_i'}\leq c_0\max\{\sigma_i,\sigma_i'\}$ and $\max\{\sigma_i/\sigma_i',\sigma_i'/\sigma_i\}\leq \eta_0$. \\
 
 Now consider a pair $P_i,P_i'$. WLOG assume that $\pi_i\leq \pi_i'$, otherwise one can switch the role of $P$ and $P'$. For all $j\neq i$,
 \begin{align}
 \norm{\mu_j-\mu_i'}&\geq c(\sigma_j'+\sigma_i') - c_0\max\{\sigma_j,\sigma_j'\}\geq c\sigma_i'+\left(\frac{c}{\eta_0}-c_0\right)\max\{\sigma_j',\sigma_j\} \geq c\sigma_i' + \left(\frac{c}{\eta_0}-c_0\right)\sigma_j\;; \\
 \norm{\mu_j-\mu_i'}&\geq c(\sigma_j+\sigma_i) - c_0\max\{\sigma_i,\sigma_i'\} \geq c\sigma_j + \left(\frac{c}{\eta_0}-c_0\right) \max\{\sigma_i,\sigma_i'\}\geq c\sigma_j + \left(\frac{c}{\eta_0}-c_0\right)\sigma_i'\;.
 \end{align}
 Therefore adding these two equation together we have 
 \begin{equation}
 \norm{\mu_j-\mu_i'} \geq \frac{1}{2}\left(c+\frac{c}{\eta_0}-c_0\right)(\sigma_j+\sigma_i')\;.
 \end{equation}
 Note that $c > \eta_0c_0$, and hence $\norm{\mu_j-\mu_i'}\geq \frac{1}{2}c_0\eta_0 (\sigma_j+\sigma_i')$. According to \lemmaref{lem:etaplusone}, there exists a set $A_j$ such that $P_i(A_j)-P_j'(A_j)>1-\rho$, where $$
 \rho=2\Phi\left(-\frac{1}{2}(c+\frac{c}{\eta_0}-c_0)\right) < 2\Phi(-\eta_0c_0/2) < 2\Phi(-c_0/2)=\pi_{\min}-2\epsilon
 $$. Then by \lemmaref{lem:rho1-rho2}, the total variation distance between $P_i$ and $P_i'$ is upper bounded by 
 \begin{equation}
 TV(P_i,P_i')\leq UB(c_0,\eta_0):=\frac{2\epsilon}{\pi_{\min}} + \frac{2(1-\pi_{\min})}{\pi_{\min}}\Phi\left(-\frac{1}{2}(c+\frac{c}{\eta_0}-c_0)\right)< 1-\pi_{\min}+2\epsilon\;.
 \label{eq:ub1}
 \end{equation}
 This further implies, by \lemmaref{lem:tv-mono},
 $
 \norm{\mu_i-\mu_i'} \leq c_1\max\{\sigma_i,\sigma_i'\},c_1 = 2\Phi^{-1}(1-\frac{1-UB(c_0,\eta_0)}{2}). 
 $
 Note that $UB(c_0,\eta_0) < 1-\pi_{\min}+2\epsilon$, therefore $c_1 < c_0$. Also by \lemmaref{lem:eta-upperbound}, 
 $
 \max\{{\sigma_i}/{\sigma_i'},{\sigma_i'}/{\sigma_i}\} \leq \eta_1,
 $
 where $\eta_1$ solves
 \begin{align}
 F_d(\frac{2\eta_1^2\log \eta_1}{\eta_1^2-1}) - F_d(\frac{2\log \eta_1}{\eta_1^2-1}) = UB(c_0,\eta_0)
 \label{eq:eta1}
 \end{align}
 Since $1-(K-1)\pi_{\min}\geq \pi_{\min}$,
 \begin{align}
 &\quad F_d(\frac{2\eta_0^2\log \eta_0}{\eta_0^2-1}) - F_d(\frac{2\log \eta_0}{\eta_0^2-1}) \\
 &=1-\frac{\pi_{\min}-2\epsilon}{1-(K-1)\pi_{\min}}+\frac{2(K-1)\pi_{\min}}{1-(K-1)\pi_{\min}}\Phi(-\frac{1}{2}\eta_0 c_0) \\
 &> 1-\frac{\pi_{\min}-2\epsilon}{1-(K-1)\pi_{\min}}+\frac{2(K-1)\pi_{\min}}{1-(K-1)\pi_{\min}}\Phi\left(-\frac{1}{2}(c+\frac{c}{\eta_0}-c_0)\right) \\
 &= 1- \frac{\pi_{\min}-2\epsilon-\left[1-(1-(K-1)\pi_{\min})\right]\rho}{1-(K-1)\pi_{\min}}\quad\quad  \left(\text{recall\ }\rho=2\Phi(-(c+{c}/{\eta_0}-c_0)/2)\right)\\
 &=1-\frac{\pi_{\min}-2\epsilon-\rho}{1-(K-1)\pi_{\min}}-\rho \\
 &\geq 1-\frac{\pi_{\min}-2\epsilon-\rho}{\pi_{\min}}-\rho \\
 & = 1- \frac{\pi_{\min}-2\epsilon}{\pi_{\min}}+\frac{1-\pi_{\min}}{\pi_{\min}}\rho \\
 &= F_d(\frac{2\eta_1^2\log \eta_1}{\eta_1^2-1}) - F_d(\frac{2\log \eta_1}{\eta_1^2-1})\;,
 \end{align}
 it holds that $\eta_1 < \eta_0$. To conclude, for now starting with $c_0,\eta_0$, we are able to provide refined upper bounds $c_1,\eta_1$. Note that by the same argument of \eqref{eq:ub1}, one can use $c_1,\eta_1$ to refine the upper bound by 
 \begin{equation}
 TV(P_i,P_i')\leq UB(c_1,\eta_1):= \frac{2\epsilon}{\pi_{\min}} + \frac{2(1-\pi_{\min})}{\pi_{\min}}\Phi\left(-\frac{1}{2}(c+\frac{c}{\eta_1}-c_1)\right)\;.
 \end{equation}
 And similarly $
 \norm{\mu_i-\mu_i'} \leq c_2\max\{\sigma_i,\sigma_i'\},c_2 = 2\Phi^{-1}(1-\frac{1-UB(c_1,\eta_1)}{2}). 
$
$
 \max\{{\sigma_i}/{\sigma_i'},{\sigma_i'}/{\sigma_i}\} \leq \eta_2,
$
where $\eta_2$ solves
\begin{align}
 F_d(\frac{2\eta_2^2\log \eta_2}{\eta_2^2-1}) - F_d(\frac{2\log \eta_2}{\eta_2^2-1}) = UB(c_1,\eta_1)\;.
 \label{eq:eta2}
\end{align}
As $UB(c_1,\eta_1)<UB(c_0,\eta_0)$, we have $c_2 < c_1,\eta_2 < \eta_1$. This procedure can be repeated. Let
\begin{equation}
 UB(c_t,\eta_t):=\frac{2\epsilon}{\pi_{\min}} + \frac{2(1-\pi_{\min})}{\pi_{\min}}\Phi\left(-\frac{1}{2}(c+\frac{c}{\eta_t}-c_t)\right)
\end{equation}
Then one can find 
\begin{equation}
 c_{t+1}=2\Phi^{-1}(1-\frac{1-UB(c_t,\eta_t)}{2}) = 2\Phi^{-1}(\frac{1}{2}+\frac{1}{2}UB(c_t,\eta_t))
 \label{eq:c-iteration}
\end{equation}
and $\eta_{t+1}$ solves 
\begin{align}
 F_d(\frac{2\eta_{t+1}^2\log \eta_{t+1}}{\eta_{t+1}^2-1}) - F_d(\frac{2\log \eta_{t+1}}{\eta_{t+1}^2-1}) = UB(c_t,\eta_t)\;
 \label{eq:etat}
\end{align}
such that 
\begin{equation}
    \norm{\mu_i-\mu_i'}\leq c_{t+1}\max\{\sigma_i,\sigma_i'\},\quad \max\{\sigma_i'/\sigma_i,\sigma_i/\sigma_i'\}\leq \eta_{t+1}\;.
\end{equation}

The L.H.S. of \eqref{eq:etat} is a strictly increasing function in $\eta_{t+1}$ and it has a continuous inverse. Therefore we conclude that there is a continuos mapping $g:[0,1] \rightarrow [0,c_0]\times [1,\eta_0]$ such that
\begin{equation}
 (c_{t+1},\eta_{t+1}) = g(c_t,\eta_t)
\end{equation}

By induction one shows that $UB(c_t,\eta_{t}) < UB(c_{t-1},\eta_{t-1})$ for $t\geq 1$, therefore, $c_{t+1}<c_t,\eta_{t+1}<\eta_t$. Hence $\{c_t\},\{\eta_t\}$ are two decreasing sequences, both lower bounded. Therefore their limits exist. Denote the limits as $(c^*,\eta^*)$, then it must be a fixed point of $g$, namely they solve \eqref{eq:cstar}. Further, $\norm{\mu_i-\mu_i'}\leq c^*\max\{\sigma_i,\sigma_i'\},\max\{\sigma_i'/\sigma_i,\sigma_i/\sigma_i'\}\leq \eta^*$.

The upper bounds on proportions $|\pi_i-\pi_i'|$ are established with the components $c^*,\eta^*$ will be proved in the next section, where we establish additional techniques to complete the proof of \theoremref{thm:main1}. 
\end{proof}

{Before we display the proof on difference in proportions, we provide the proof to corollary \ref{cor:cover-by-single}}.\\
\begin{proof}[Proof of Corollary \ref{cor:cover-by-single}]
 Suppose $P'$ is a single spherical Gaussians such that $TV(P,P')\leq 2\epsilon$. Consider two new mixtures of spherical Gaussians 
 \begin{align}
 \tilde{P} &= \frac{1}{2} P + \frac{1}{2} Q \\
 \tilde{P}' &= \frac{1}{2} P' + \frac{1}{2} Q 
 \end{align}
 where $Q$ is a sufficiently large far away Gaussian components. To be specfic, for any component $N_d(\mu_i,\sigma_i^2I_d)$ in $P$ or $P'$, $Q=N_d(\mu,\sigma^2I_d)$ is selected such that $\norm{\mu-\mu_i}\geq c\max\{\sigma,\sigma_i\}$. Therefore $\tilde{P}\in \mathcal{M}(K+1,\pi_{\min}/2,c)$ and $\tilde{P'}\in\mathcal{M}(2,\pi_{\min}/2,c)$. However since $TV(\tilde{P},\tilde{P'})\geq 2\epsilon/2=\epsilon$, and the maximal proportion of $\tilde{P},\tilde{P}'$ is upper bounded by $1/2$. Let $\eta_0'$ be defined as the solution to 
 \begin{equation}
 1-\pi_{\min}+\epsilon+2\Phi(-\frac{1}{2}\eta_0' c_0') = F_d(\frac{2\eta_0'^2\log \eta_0'}{\eta_0'^2-1}) - F_d(\frac{2\log \eta_0'}{\eta_0'^2-1})
 \end{equation}
 where $c_0'=2\Phi^{-1}(1-\frac{\pi_{\min}-\epsilon}{4})$, then if $c>c_0'\eta_0'$, according to \theoremref{thm:unique-matching}, $K+1 = 2$, which implies that $K=1$ and this is a contradiction.
\end{proof}

\subsection{Difference in Proportions}
Finally given $\eta^*,c^*$ established previously, we will complete the proof of \theoremref{thm:main1} by upper bounding the differences in proportions. For a corresponding pair $\pi_i,\pi_i'$ where $\norm{\mu_i-\mu_i'}\leq c^*\max\{\sigma_i,\sigma_i'\}$ and $\max\{\sigma_i',\sigma_i\}/\min\{\sigma_i',\sigma_i\}\leq \eta^*$, we upper bound $|\pi_i-\pi_i'|$ by constructing the set $A$ such that we can apply \lemmaref{lem:diff-proportion}.

To start with, we introduce a geometric lemma.

\begin{lemma}
 \newcommand{\bmx}{x}
\newcommand{\bmy}{y}
\newcommand{\bmxt}{\tilde{x}}
\newcommand{\bmyt}{\tilde{y}}
 Let $\bmx_1,\bmx_2,\bmy_1,\bmy_2$ be for different points in $\mathbb{R}^d$, where 
 \begin{align}
 \tilde{\bmx} &= \alpha \bmx_1 + (1-\alpha) \bmx_2\;, \\
 \tilde{\bmy} &= \beta\bmy_1 + (1-\beta) \bmy_2\;.
 \end{align} 
 Then 
 \begin{align*}
 \norm{\tilde{\bmx}-\tilde{\bmy}}^2 & =\alpha\beta\norm{\bmx_1-\bmy_1}^2+(1-\alpha)(1-\beta)\norm{\bmx_2-\bmy_2}^2+(1-\alpha)\beta \norm{\bmx_2-\bmy_1}^2+\alpha(1-\beta)\norm{\bmx_1-\bmy_2}^2 \\
 &\quad -\alpha(1-\alpha)\norm{\bmx_1-\bmx_2}^2-\beta(1-\beta)\norm{\bmy_1-\bmy_2}^2
 \end{align*}
 \label{lem:mass-center}
\end{lemma}

The proof of this lemma is postponed to \sectionref{sec:proof-aux}. Then consider two pairs of components $P_i, P_i'$ and $P_j, P_j'$, one can find a hyperplane that is far away from all four centers. 

\begin{lemma}
 Let $P_i=N_d(\mu_i,\sigma_i^2 I_d),i=1,2$ and $P_i'=N_d(\mu_i',\sigma_i'^2 I_d)$ be two pairs of spherical Gaussian distributions such that with three constants $c^*\geq 0,\eta^*\geq 1, c\geq c^*\eta^*$ such that
 \begin{alignat}{4}
 &\max\{\frac{\sigma_1}{\sigma_1'},\frac{\sigma_1'}{\sigma_1}\}\leq \eta^*,&\quad &\max\{\frac{\sigma_2}{\sigma_2'},\frac{\sigma_2'}{\sigma_2}\}\leq \eta^*\\
 &\norm{\mu_1-\mu_1'}\leq \frac{c^*}{2}(\sigma_1+\sigma_1'),&\quad &\norm{\mu_2-\mu_1}\geq c(\sigma_1 + \sigma_2) \\
 &\norm{\mu_2-\mu_2'}\leq \frac{c^*}{2}(\sigma_2+\sigma_2'),&\quad &\norm{\mu_1'-\mu_2'}\geq c(\sigma_1'+\sigma_2')
 \end{alignat}
 There exists a hyperplane $H$ such that the distance 
 \begin{align}
 dist(\mu_i,H) \geq C(c^*,\eta^*,c)\sigma_i ,\quad dist(\mu_i',H)\geq C(c^*,\eta^*,c)\sigma_i' ,\quad i=1,2
 \end{align}
 where
 \begin{equation}
 C(c^*,\eta^*,c) = \sqrt{\frac{c^2}{2(\eta^*)^2}+\frac{1}{2\eta^*}(c-\frac{c^*}{2})^2-\frac{(c^*)^2(1+\eta^*)^2}{16(\eta^*)^2}}-\frac{c^*}{2}
 \label{eq:svm-halfplane}
 \end{equation}
 \label{lem:twopairs}
\end{lemma}

\begin{proof}[Proof of \lemmaref{lem:twopairs}]
Consider two points 
\begin{equation}
 \tilde{\mu}_i = \frac{\sigma_i}{\sigma_i+\sigma_i'}\mu_i + \frac{\sigma_i'}{\sigma_i+\sigma_i'}\mu_i',\quad i=1,2\;,
\end{equation}
Note that for both $i=1,2$,
\begin{equation}
 \norm{\tilde{\mu}_i - \mu_i} = \frac{\sigma_i'}{\sigma_i+\sigma_i'}\norm{\mu_i'-\mu_i}\leq \frac{c^*\max\{\sigma_i,\sigma_i'\}}{2},\quad\norm{\tilde{\mu}_i - \mu_i'} = \frac{\sigma_i}{\sigma_i+\sigma_i'}\norm{\mu_i'-\mu_i}\leq \frac{c^*\max\{\sigma_i,\sigma_i'\}}{2}
\end{equation}
that is, $\mu_i,\mu_i'$ both lies in a ball centered at $\tilde{\mu}_i$ with radius $c^*\max\{\sigma_i,\sigma_i'\}/2$. The hyperplane we consider is 
\begin{equation}
 H:\left\{x\in\mathbb{R}^d: \langle x-\tilde{\mu_1},\tilde{\mu}_2-\tilde{\mu}_1 \rangle = \frac{\sigma_1+\sigma_1'}{\sigma_1+\sigma_1'+\sigma_2+\sigma_2'}\norm{\tilde{\mu}_2-\tilde{\mu}_1}\right\}
\end{equation}

By \lemmaref{lem:mass-center}, the distance 
\begin{align*}
 \norm{\tilde{\mu}_1-\tilde{\mu}_2}^2 &= \underbrace{\frac{\sigma_1\sigma_2}{(\sigma_1+\sigma_1')(\sigma_2+\sigma_2')}\norm{\mu_1-\mu_2}^2+\frac{\sigma_1'\sigma_2'}{(\sigma_1+\sigma_1')(\sigma_2+\sigma_2')}\norm{\mu_1'-\mu_2'}^2}_{LB_1}\\
 &\quad+ \underbrace{\frac{\sigma_1\sigma_2'}{(\sigma_1+\sigma_1')(\sigma_2+\sigma_2')}\norm{\mu_1-\mu_2'}^2+\frac{\sigma_1'\sigma_2}{(\sigma_1+\sigma_1')(\sigma_2+\sigma_2')}\norm{\mu_1'-\mu_2}^2}_{LB_2} \\
 &\quad -\underbrace{\frac{\sigma_1\sigma_1'}{(\sigma_1+\sigma_1')^2}\norm{\mu_1-\mu_1'}^2-\frac{\sigma_2\sigma_2'}{(\sigma_2+\sigma_2')^2}\norm{\mu_2-\mu_2'}^2}_{UB}
\end{align*}
We will establish a lower bound on $\norm{\tilde{\mu}_1-\tilde{\mu}_2}$ by separately lower bounding $LB_1,LB_2$ and $UB$.\\

First, by the seperation condition
\begin{align}
{LB_1} &\geq c^2\frac{\sigma_1\sigma_2(\sigma_1+\sigma_2)^2+\sigma_1'\sigma_2'(\sigma_1'+\sigma_2')^2}{(\sigma_1+\sigma_1')(\sigma_2+\sigma_2')(\sigma_1+\sigma_2+\sigma_1'+\sigma_2')^2}(\sigma_1+\sigma_2+\sigma_1'+\sigma_2')^2
\end{align}
Note that $\max\{\sigma_i/\sigma_i',\sigma_i'/\sigma_i\}\leq \eta^*,i=1,2$, we denote 
\begin{align}
 \alpha = \frac{\sigma_1}{\sigma_1+\sigma_1'} \;,\quad
 \beta = \frac{\sigma_2'}{\sigma_2+\sigma_2'} \;,\quad
 \gamma = \frac{\sigma_1+\sigma_2}{\sigma_1+\sigma_1'+\sigma_2+\sigma_2'}\;,
\end{align}
then since $1/({1+\eta^*})\leq \alpha,\beta,\gamma \leq {\eta^*}/({1+\eta^*})$, it holds that 
\begin{align}
 \frac{LB_1}{c^2(\sigma_1+\sigma_2+\sigma_1'+\sigma_2')^2}&=\alpha\beta\gamma^2 + (1-\alpha)(1-\beta)(1-\gamma)^2 \\
 &\geq \left(\frac{1}{1+\eta^*}\right)^2[\gamma^2 + (1-\gamma)^2] \\
 &\geq \frac{1}{2}\left(\frac{1}{1+\eta^*}\right)^2\;.
\end{align}
namely
\begin{equation}
 LB_1 \geq \frac{c^2}{2}\left(\frac{1}{1+\eta^*}\right)^2(\sigma_1+\sigma_2+\sigma_1'+\sigma_2')^2
\end{equation}

For the second term, start by observing that 
\begin{equation}
 \norm{\mu_1-\mu_2'} \geq c(\sigma_1+\sigma_2)-\frac{c^*}{2}(\sigma_2+\sigma_2'),\quad \norm{\mu_1-\mu_2'} \geq c(\sigma_1'+\sigma_2')-\frac{c^*}{2}(\sigma_1+\sigma_1')
\end{equation}
Adding these two equations we obtain that 
\begin{equation}
 \norm{\mu_1-\mu_2'}\geq \frac{1}{2}(c-\frac{c^*}{2})(\sigma_1+\sigma_1'+\sigma_2+\sigma_2')\;.
\end{equation}
and similarly $\norm{\mu_1'-\mu_2}\geq (c-\frac{c^*}{2})(\sigma_1+\sigma_1'+\sigma_2+\sigma_2')/2$.
Hence 
\begin{align}
 \frac{LB_2}{(\sigma_1+\sigma_1'+\sigma_2+\sigma_2')^2} &\geq \frac{1}{4}(c-\frac{c^*}{2})^2[\alpha(1-\beta)+\beta(1-\alpha)]\;.
\end{align}

Write $g(\alpha,\beta)=\alpha(1-\beta)+\beta(1-\alpha)=\beta(1-2\alpha)+\alpha$. 
\begin{itemize}
 \item If $\alpha = 1/2$ then $g(\alpha,\beta)=1/2$.
 \item If $1/(1+\eta^*)\leq \alpha < 1/2$, then $g(\alpha,\beta) \geq g(\alpha,1/(1+\eta^*))$ since it is an affine function in $\beta$ with positive linear term. When $\beta = 1/(1+\eta^*)<1/2$, by same argument we know that $g(\alpha,\beta)$ is an affine function in $\alpha$ with positive linear term. Hence the minimum is 
 \begin{equation}
 g(\frac{1}{1+\eta^*},\frac{1}{1+\eta^*})=\frac{2\eta^*}{(1+\eta^*)^2}\;.
 \end{equation}
 \item If $1/2 <\alpha \leq \eta^*/(1+\eta^*)$, the same arguments show that the minimum is 
 \begin{equation}
 g(\frac{\eta^*}{1+\eta^*},\frac{\eta^*}{1+\eta^*})=\frac{2\eta^*}{(1+\eta^*)^2}\;.
 \end{equation}
\end{itemize}
As a conclution we have 
\begin{equation}
 LB_2 \geq \frac{1}{4}\frac{2\eta^*}{(1+\eta^*)^2}(c-\frac{c^*}{2})^2(\sigma_1+\sigma_1'+\sigma_2+\sigma_2')^2
\end{equation}

Finally, we establish an upper bound for $UB$, 

\begin{align}
 UB &= \frac{\sigma_1\sigma_1'}{(\sigma_1+\sigma_1')^2}\norm{\mu_1-\mu_1'}^2-\frac{\sigma_2\sigma_2'}{(\sigma_2+\sigma_2')^2}\norm{\mu_2-\mu_2'}^2 \\
 &\leq \frac{(c^*)^2}{16}(\sigma_1+\sigma_1')^2 + \frac{(c^*)^2}{16}(\sigma_2+\sigma_2')^2 \\
 &\leq \frac{(c^*)^2}{16}(\sigma_1 + \sigma_1' +\sigma_2 +\sigma_2')^2
\end{align}

Therefore 
\begin{equation}
 \norm{\tilde{\mu}_1 - \tilde{\mu}_2}\geq \sqrt{\frac{c^2}{2}\left(\frac{1}{1+\eta^*}\right)^2+\frac{1}{4}\frac{2\eta^*}{(1+\eta^*)^2}(c-\frac{c^*}{2})^2-\frac{(c^*)^2}{16}}(\sigma_1+\sigma_1'+\sigma_2+\sigma_2')
\end{equation}

The distance of $\mu_1,\mu_1'$ to $H$ is then lower bounded by 
\begin{align}
 dist(\mu_1,H) &\geq \norm{\tilde{\mu}_1 - \tilde{\mu}_2}\frac{\sigma_1+\sigma_1'}{\sigma_1+\sigma_1'+\sigma_2+\sigma_2'} - \frac{c^*}{2}\max\{\sigma_1,\sigma_1'\} \\
 &\geq \sqrt{\frac{c^2}{2}\left(\frac{1}{1+\eta^*}\right)^2+\frac{1}{4}\frac{2\eta^*}{(1+\eta^*)^2}(c-\frac{c^*}{2})^2-\frac{(c^*)^2}{16}} (\sigma_1 + \sigma_1') - \frac{c^*}{2}\max\{\sigma_1,\sigma_1'\} \\
 &\geq \sqrt{\frac{c^2}{2}\left(\frac{1}{1+\eta^*}\right)^2+\frac{1}{4}\frac{2\eta^*}{(1+\eta^*)^2}(c-\frac{c^*}{2})^2-\frac{(c^*)^2}{16}} (\frac{1}{\eta^*}+1)\max\{\sigma_1,\sigma_1'\} - \frac{c^*}{2}\max\{\sigma_1,\sigma_1'\} \\
 &\geq \left(\sqrt{\frac{c^2}{2(\eta^*)^2}+\frac{1}{2\eta^*}(c-\frac{c^*}{2})^2-\frac{(c^*)^2(1+\eta^*)^2}{16(\eta^*)^2}}-\frac{c^*}{2} \right)\max\{\sigma_1,\sigma_1'\}\\
 &:=C(c^*,\eta^*,c)\max\{\sigma_1,\sigma_1'\}
 \label{eq:distmu1-H}
\end{align}
Note that since $c>c^*\eta^*$ and $\eta^* \geq 1$,
\begin{align}
 &\quad \frac{c^2}{2(\eta^*)^2}+\frac{1}{2\eta^*}(c-\frac{c^*}{2})^2-\frac{(c^*)^2(1+\eta^*)^2}{16(\eta^*)^2} - \frac{(c^*)^2}{4} \\
 &> \frac{(c^*)^2}{2} + \frac{(c^*)^2(\eta^*-\frac{1}{2})^2}{2\eta^*} - \frac{(c^*)^2(1+\eta^*)^2}{16(\eta^*)^2} - \frac{(c^*)^2}{4} \\
 &=(c^*)^2\left[\frac{4(\eta^*)^2+8\eta^*(\eta^*-\frac{1}{2})^2-(1+\eta^*)^2}{16(\eta^*)^2}\right]\\
 &=\left(\frac{c^*}{4\eta^*}\right)^2\left[8(\eta^*)^3-5(\eta^*)^2-1\right] > 0
\end{align}
The coefficient $C(c^*,\eta^*,c)$ in \eqref{eq:distmu1-H} is positive, so 
\begin{align}
 dist(\mu_1,H) &\geq \left(\sqrt{\frac{c^2}{2(\eta^*)^2}+\frac{1}{2\eta^*}(c-\frac{c^*}{2})^2-\frac{(c^*)^2(1+\eta^*)^2}{16(\eta^*)^2}}-\frac{c^*}{2} \right)\sigma_1 \\
 dist(\mu_1',H) &\geq \left(\sqrt{\frac{c^2}{2(\eta^*)^2}+\frac{1}{2\eta^*}(c-\frac{c^*}{2})^2-\frac{(c^*)^2(1+\eta^*)^2}{16(\eta^*)^2}}-\frac{c^*}{2} \right)\sigma_1'
\end{align}
This also shows that $\mu_1,\mu_1'$ are both with the same side of $H$ as $\tilde{\mu}_1$. $\mu_2,\mu_2'$ can be similarly shown to be with the same side of $H$ as $\tilde{\mu}_2$, which is different from $\mu_1,\mu_1'$ and 
and the other two distance $dist(\mu_2,H),dist(\mu_2',H)$ can be lower bounded similarly.
\end{proof}

\begin{proof}[Proof of \theoremref{thm:main1}: Difference in Proportions]
 Now we can finish the proof to \theoremref{thm:main1}. Let $P_i,P_i'$ be two corresponding pairs such that 
 \begin{equation}
 \max\{\frac{\sigma_i}{\sigma_i'},\frac{\sigma_i'}{\sigma_i}\}\leq \eta^*,\quad \norm{\mu_i-\mu_i'}\leq \frac{c^*}{2}(\sigma_i+\sigma_i')\;.
 \end{equation}
 Here the center distance is slightly different. In fact for the corresponding pairs, it holds that 
 \begin{equation}
 TV(P_i,P_i') \leq \frac{2\epsilon}{\pi_{\min}}+\frac{2(1-\pi_{\min})}{\pi_{\min}}\Phi(-\frac{1}{2}\left(c+\frac{c}{\eta^*}-c^*\right))=1-2\Phi(-\frac{c^*}{2})\;.
 \end{equation}
 According to \lemmaref{lem:etaplusone}, it also holds that $\norm{\mu_i-\mu_i'}\leq c^*(\sigma_i+\sigma_i')/2$. Consider an arbitrary different corresponding pair $P_j,P_j'$, then $P_i,P_i',P_j,P_j'$ satisfy the conditions of \lemmaref{lem:twopairs}, hence we can select a hyperplane $H$ such that $P_i,P_i'$ and $P_j,P_j'$ are on different side of $H$, and their distance are lower bounded. Let $A_j$ be the halfspace determined by $H$ and including $P_j,P_j'$, then 
 \begin{equation}
 \max\{P_i(A_j^\complement),P_i'(A_j^\complement),P_j(A_j),P_j'(A_j)\}\leq \Phi(-C(c,c^*,\eta^*))
 \end{equation}
 where $C(c,c^*,\eta^*)$ is determined in \eqref{eq:svm-halfplane}. Now take $A=\cap_{j\neq i}A_j$, then 
 \begin{alignat}{3}
 P_i(A) &= P_i((\cup_{j\neq i}A_j^\complement)^\complement) \geq 1- \sum_{j\neq i}P_i(A_j^\complement)\geq 1-(K-1)\Phi(-C(c,c^*,\eta^*))\;,\\
 P_i'(A)&= P_i'((\cup_{j\neq i}A_j^\complement)^\complement) \geq 1- \sum_{j\neq i}P_i'(A_j^\complement) \geq 1-(K-1)\Phi(-C(c,c^*,\eta^*))\;;\\
 P_j(A) &\leq P_j(A_j)\leq \Phi(-C(c,c^*,\eta^*))\\
 P_j'(A) &\leq P_j'(A_j) \leq \Phi(-C(c,c^*,\eta^*))
 \end{alignat}
And the result is given by \lemmaref{lem:diff-proportion}.
\end{proof}

\section{Auxillary Lemmas}
\label{sec:proof-aux}

\subsection{Lemmas in \sectionref{sec:setup}}

\begin{lemma}
 Let $F_d$ be the cumulative distribution function of a Gamma($\frac{d}{2},\frac{d}{2}$) distribution, then for any $\eta\geq 1$, 
 \begin{equation}
 \text{R.H.S. of \eqref{eq:eta0-diff}}=F_d(\frac{2\eta^2\log \eta}{\eta^2-1}) - F_d(\frac{2\log \eta}{\eta^2-1}) \geq 1-\left(\frac{2\eta}{\eta^2+1}\right)^{\frac{d}{2}}\;,
 \end{equation}
 \label{lem:est-eta-function}
\end{lemma}

\begin{proof}[Proof of \lemmaref{lem:est-eta-function}]
 Consider two distributions $P_1 = N_d(0,I_d)$ and $P_2 = N_d(0,\eta^2 I_d)$ with $\eta\geq 1$ and the set 
 \begin{equation}
 A=\{x\in\mathbb{R}^d: p_1(x) \geq p_2(x)\}=\left\{x\in\mathbb{R}^d: \norm{x}\leq \frac{\eta}{\eta^2-1}\sqrt{2d(\eta^2-1)\log\eta}\right\}\;.
 \end{equation}
 where $p_i,i=1,2$ is the density of $P_i$. Then $P_1(A)-P_2(A)$ is the R.H.S. expression in \eqref{eq:eta0-diff}. Further, it is the total variation distance between $P_1$ and $P_2$. Well known result states that total variation distance is lower bounded by Hellinger distance \citep{Gibbs2002probDist}, therefore it holds that 
 \begin{equation}
 \text{R.H.S. of \eqref{eq:eta0-diff}} \geq \frac{1}{2}\int(\sqrt{p_1}-\sqrt{p_2})^2 =1 - \left(\frac{2\eta}{\eta^2+1}\right)^{\frac{d}{2}}\;.
 \end{equation} 
\end{proof}

\subsection{Lemmas in \sectionref{sec:proof-main}}
\begin{lemma}
 If there exists a set $A$ such that $P_1(A)-P_2(A) \geq 1-\rho\geq 0$. Then for any weights $0 \leq \pi_1,\pi_2 \leq 1$:
 \begin{equation}
 \pi_1-\max\{\pi_1,\pi_2\}\rho \geq \pi_1 P_1(A) - \pi_2 P_2(A) \leq \pi_1
 \end{equation}
 \label{lem:weightedTotalVariation}
\end{lemma}

\begin{proof}[Proof of \lemmaref{lem:weightedTotalVariation}]
 Note that
 \begin{align*}
 \pi_1P_1(A)-\pi_2P_2(A)\geq  \pi_1(P_2(A)+1-\rho)-\pi_2P_2(A)
 =\pi_1+(\pi_1-\pi_2)P_2(A)-\pi_1\rho \;.
 \end{align*}
 View this as an affine function in $P_2(A)$. Since $0\leq P_2(A)\leq \rho$, we have 
 \begin{equation}
 \pi_1P_1(A)-\pi_2P_2(A) \geq \min\{\pi_1+(\pi_1-\pi_2)\cdot 0-\pi_1\rho,\pi_1+(\pi_1-\pi_2)\rho-\pi_1\rho\}=\pi_1-\max\{\pi_1,\pi_2\}\rho\;.
 \end{equation}
\end{proof}

\begin{lemma}
    When $x > 1$, $f(x)=x\log x/(x-1)$ is increasing in $x$ and $g(x)=\log x/(x-1)$ is decreasing in $x$.
    \label{lem:mono-eta-function}
   \end{lemma}
   
   \begin{proof}[Proof of \lemmaref{lem:mono-eta-function}]
    Take derivative 
    \begin{align}
    f'(x) &= \frac{(x-1)(1+\log x)-x\log x}{(x-1)^2} = \frac{x-1-\log x}{(x-1)^2} \geq 0 \\
    g'(x) &= \frac{(x-1)/x-\log x}{(x-1)^2} = \frac{1}{(x-1)^2}[1-\frac{1}{x}+\log \frac{1}{x}]\leq 0
    \end{align}
   \end{proof}

\begin{proof}[Proof of\lemmaref{lem:mass-center}]
 \newcommand{\bmx}{x}
\newcommand{\bmy}{y}
\newcommand{\bmxt}{\tilde{x}}
\newcommand{\bmyt}{\tilde{y}}
 This lemma can be proved by direct computation. First, note that
 \begin{equation}
 \norm{\bmy_1-\bmy_2}^2 = \norm{\bmy_1 - \bmx_1}^2 + \norm{\bmy_2-\bmx_1}^2 - 2\langle \bmy_1 - \bmx_1, \bmy_2 - \bmx_1 \rangle
 \end{equation}
Then 
 \begin{alignat}{1}
 \notag
 &\quad \norm{\bmx_1 - \bmyt}^2 \\
 &= \norm{\beta(\bmx_1 - \bmy_1) + (1-\beta)(\bmx_1 - \bmy_2)}^2 \\
 &= \beta^2 \norm{\bmx_1-\bmy_1}^2 + (1-\beta)^2\norm{\bmx_1-\bmy_2}^2 + 2 \beta(1-\beta) \langle \bmx_1 - \bmy_1, \bmx_1 - \bmy_2 \rangle \\
 &= \beta^2 \norm{\bmx_1-\bmy_1}^2 + (1-\beta)^2\norm{\bmx_1-\bmy_2}^2 + 2\beta(1-\beta) \frac{\norm{\bmx_1-\bmy_1}^2 + \norm{\bmx_1-\bmy_2}^2 - \norm{\bmy_1-\bmy_2}^2}{2} \\
 &= \beta \norm{\bmx_1-\bmy_1}^2 + (1-\beta)\norm{\bmx_1-\bmy_2}^2-\beta(1-\beta)\norm{\bmy_1-\bmy_2}^2\;.
 \end{alignat}
 Similarly we have 
 \begin{equation}
 \norm{\bmx_2-\bmyt}^2 = \beta \norm{\bmx_2-\bmy_1}^2 + (1-\beta)\norm{\bmx_2 - \bmy_1}^2-\beta(1-\beta)\norm{\bmy_1-\bmy_2}^2\;.
 \end{equation}
 Combining all above together, it holds that 
 \begin{align}
 \notag 
 &\ \quad\norm{\bmxt-\bmyt}^2\\
 &= \alpha \norm{\bmyt-\bmx_1}^2 + (1-\alpha)\norm{\bmyt-\bmx_2}^2-\alpha(1-\alpha)\norm{\bmx_1-\bmx_2}^2 \\
 \notag
 &=\alpha\beta\norm{\bmx_1-\bmy_1}^2+(1-\alpha)(1-\beta)\norm{\bmx_2-\bmy_2}^2+(1-\alpha)\beta \norm{\bmx_2-\bmy_1}^2+\alpha(1-\beta)\norm{\bmx_1-\bmy_2}^2 \\
 &\quad -\alpha(1-\alpha)\norm{\bmx_1-\bmx_2}^2-\beta(1-\beta)\norm{\bmy_1-\bmy_2}^2
 \end{align}
\end{proof}

\end{document}